\documentclass[11pt, oneside, english]{article}
\usepackage[english]{babel}
\usepackage[utf8x]{inputenc}
\usepackage[T1]{fontenc}

\usepackage{geometry}
\usepackage{amsmath,amsfonts,amsthm,amssymb,mathrsfs,dsfont,bbm} % Math packages
\usepackage{mathtools}
\usepackage{graphicx}
\usepackage[colorinlistoftodos]{todonotes}
\usepackage[colorlinks=true, allcolors=blue]{hyperref}
\usepackage[capitalize,nameinlink]{cleveref}
\usepackage{verbatim}
\usepackage{algpseudocode,algorithm,algorithmicx}
\allowdisplaybreaks

\newcommand{\N}{\mathbb{N}}

\newcommand{\bE}{\mathbb{E}}
\newcommand{\R}{\mathbb{R}}
\newcommand{\F}{\mathcal{F}}

\newcommand{\Cov}{\mathrm{Cov}}
\newcommand{\Var}{\mathrm{Var}}
\newcommand{\E}{\bE}      % Expectation

\newcommand{\KL}{\mathop{\bf KL\/}}
\newcommand{\TV}{{\bf TV}}

\DeclareMathOperator*{\argmin}{arg\,min}
\DeclareMathOperator{\Ber}{Ber}
\DeclareMathOperator{\poly}{poly}

\newtheorem{theorem}{Theorem}[section]

\newtheorem*{namedtheorem}{\theoremname}
\newcommand{\theoremname}{testing}

\newtheorem{lemma}[theorem]{Lemma}

\newtheorem{proposition}[theorem]{Proposition}

\newtheorem{corollary}[theorem]{Corollary}

\newtheorem{conjecture}[theorem]{Conjecture}
\newtheorem*{question*}{Question}

\theoremstyle{definition}
\newtheorem{definition}[theorem]{Definition}

\newtheorem{remark}[theorem]{Remark}

\theoremstyle{plain}

\begin{document}
\title{Mean-field approximation, convex hierarchies, and the 
optimality of correlation rounding: a unified perspective}
\author{Vishesh Jain \thanks{Massachusetts Institute of Technology. Department of Mathematics. Email: {\tt visheshj@mit.edu}. Research is partially supported by NSF CCF 1665252 and DMS-1737944 and ONR N00014-17-1-2598.}
\and Frederic Koehler \thanks{Massachusetts Institute of Technology. Department of Mathematics. Email: {\tt fkoehler@mit.edu}. Research is partially supported by NSF Large CCF-1565235 and Ankur Moitra's David and Lucile Packard Fellowship.}
\and Andrej Risteski \thanks{Massachusetts Institute of Technology. Department of Mathematics and IDSS. Email: {\tt risteski@mit.edu}.}}
\date{}
\maketitle
\abstract{ 
The free energy is a key quantity of interest in Ising models, but unfortunately, computing it in general is computationally intractable. Two popular (variational) approximation schemes for estimating the free energy of general Ising models (in particular, even in regimes where correlation decay does not hold) are: (i) the \emph{mean-field approximation}
with roots in statistical physics, which estimates the free energy from below, and (ii) \emph{hierarchies of convex relaxations} with roots in theoretical computer science, which estimate the free energy from above. We show, surprisingly, that the tight regime for both methods to compute the free energy to leading order is \emph{identical}. 

More precisely, we show that the mean-field approximation is within $O((n\|J\|_{F})^{2/3})$ of the free energy, where $\|J\|_F$ denotes the Frobenius norm of the interaction matrix of the Ising model. This simultaneously subsumes both the breakthrough work of Basak and Mukherjee, who showed the tight result that the mean-field approximation is within $o(n)$ whenever $\|J\|_{F} = o(\sqrt{n})$, as well as the work of Jain, Koehler, and Mossel, who gave the previously best known non-asymptotic bound of $O((n\|J\|_{F})^{2/3}\log^{1/3}(n\|J\|_{F}))$.  
%the regime in which mean-field approximation gives a non-trivial guarantee was recently identified in breakthrough work of [Basak and Mukherjee, 2017] and tightened by [Jain et al., 2018a]
We give a simple, \emph{algorithmic} proof of this result using a convex relaxation proposed by Risteski based on the Sherali-Adams hierarchy, automatically giving \emph{sub-exponential time approximation schemes for the free energy} in this entire regime. Our algorithmic result is tight under Gap-ETH.

%correct regime ($\|J\|_F^2 = o(n)$) for mean-field approximation to be accurate was recently identified in breakthrough work of Basak and Mukherjee building on the nonlinear large deviations framework of Chatterjee and Dembo.  We give a simple, algorithmic proof of this result using correlation rounding and the Sherali-Adams hierarchy, automatically giving \emph{sub-exponential time approximation schemes for the free energy} in this entire regime. Our algorithmic result is tight under Gap-ETH. 

We furthermore combine our techniques with spin glass theory to prove (in a strong sense) the \emph{optimality of correlation rounding}, refuting a recent conjecture of Allen, O'Donnell, and Zhou. Finally, we give the tight generalization of all of these results to $k$-MRFs, capturing as a special case previous work on approximating MAX-$k$-CSP. 
}
\section{Introduction}
One of the most widely studied probabilistic models in statistical physics and machine learning is the \emph{Ising model}, which is a probability distribution on the hypercube $\{\pm1\}^n$ of the form
\[ P[X = x] := \frac{1}{Z} \exp\left(\sum_{i < j} J_{i,j} x_i x_j\right) = \frac{1}{Z} \exp\left(\frac{1}{2} x^T J x\right), \]
where $\{J_{i,j}\}_{i,j\in\{1,\dots,n\}}$ are the entries of
an arbitrary real, symmetric matrix with zeros on the diagonal. The distribution $P$ is also referred to as the \emph{Boltzmann distribution} or \emph{Gibbs measure}. 
The key quantity of interest is the normalizing constant $$Z:=\sum_{x\in\{\pm1\}^{n}}\exp\left(\sum_{i < j}J_{i,j}x_{i}x_{j}\right),$$
known as the \emph{partition function }of the Ising model, and its logarithm, $\F := \log{Z}$, known as the \emph{free energy}. The reason these are important is that one can easily extract from them many other quantities of interest,  most notably the values of the marginals (probabilities like $P[X_i = x_i]$), phase transitions in the behavior of the distribution (e.g. existence of long-range correlations), and many others. 
\sloppypar{
Although originally introduced in statistical physics, Ising models and their generalizations have also found a wide range of applications in many different areas like statistics, computer science, combinatorics and machine learning (see, e.g., the references and discussion in \cite{basak2017universality, borgs2012convergent,wainwright-jordan-variational}). Consequently, various different algorithmic and analytic approaches to computing and/or approximating the free energy have been developed.
}

We should note at the outset that the partition function is both analytically and computationally intractable: closed form expressions for the partition function are extremely hard to derive (even for the Ising model on the standard $3$-dimensional lattice), %remains one of the most outstanding problems in statistical physics. 
%and n that exactly computing the partition function of an Ising model with $J$ the adjacency matrix of a nonplanar graph is NP-hard (\cite{Istrail2000StatisticalMT}), 
and even crudely approximating the partition function multiplicatively is NP-hard, even in the case of graphs with degrees bounded by a small constant (see \cite{sly-sun}). 

Nevertheless, there are a plethora of approaches to approximating the partition function -- both for the purposes of deriving structural results, and for designing efficient algorithms. A major group of approaches consist of so-called \emph{variational methods}, which proceed by writing a variational expression for the free energy, and then modifying the resulting optimization problem in some way so as to make it tractable. More concretely, one can write the free energy using the \emph{Gibbs variational principle} as 
\begin{equation}\label{eqn:gibbs}
 \F = \max_{\mu} \left[\sum_{i < j} J_{ij} \E_{\mu}[X_i X_j] + H(\mu)\right], 
\end{equation}
where $\mu$ ranges over all probability distributions on the Boolean hypercube. This can be seen by noting that 
\begin{equation}
\label{eqn:free-energy-KL}
\KL(\mu ||P)=\F - \sum_{i < j} J_{ij} \E_{\mu}[X_i X_j] - H(\mu)
\end{equation}
and recalling that $\KL(\mu ||P) \geq 0$ with equality if and only if $\mu = P$.

Of course, the polytope of distributions $\mu$ is intractable to optimize over. Two popular approaches for handling this are: \\
\begin{enumerate}
\item \emph{Mean-field approximation}: instead of optimizing over all distributions, one optimizes over \emph{product distributions}, thereby obtaining a lower bound on $\F$. In other words, we define the (mean-field) \emph{variational free energy} by
\[ \F^* := \max_{x \in [-1,1]^n} \left[\sum_{i < j} J_{ij}
      x_i x_j + \sum_i H\left(\frac{x_i +
        1}{2}\right)\right]. \] 
Indeed, if $\bar{x} = (\bar{x}_1,\dots,\bar{x}_n)$ is the optimizer in the above definition, the product distribution $\nu$ on the Boolean hypercube, with the $i^{th}$ coordinate having expectation $\bar{x}_i$ minimizes $\KL(\mu||P)$ among all product distributions $\mu$. 

This approach originated in the physics literature where it was used to great success in several cases, but from the point of view of algorithms it is \emph{a priori} problematic: it's not clear this problem is any easier to solve, as 
the resulting optimization problem is highly non-convex. 
  
\item \emph{Moment-based convex relaxations}: instead of optimizing over distributions, one optimizes over a ``relaxation'' (enlarging) of the polytope of distributions, thereby obtaining an upper bound on $\F$. There are systematic ways to do this, giving rise to \emph{hierarchies} of convex relaxations (see, e.g. \cite{barak2011rounding}). This approach is very natural and common in theoretical computer science, since the optimization problem is convex, hence efficiently solvable, although quantifying the quality of the relaxation is usually more difficult.    
\end{enumerate}

\emph{A priori} these two approaches seem unrelated -- indeed, the way they modify the variational problem is almost opposite. In this paper, we provide a unified perspective on these two approaches: for example, we show that the tight parameter regime where mean-field approximation and Sherali-Adams based approaches (even for classical MAX-$k$-CSP) give nontrivial guarantees is \emph{identical}.

More precisely, we prove the following results. \\ 
\begin{enumerate}
\item {\bf Simple and optimal mean-field bounds via rounding}: We obtain the optimal bounds on the quality of the mean-field approximation in a simple and elegant way. In particular, we show that there is a simple \emph{rounding} procedure which directly extracts a product distribution from the true Gibbs measure, and whose output is easy to analyze.
\begin{comment}
via \emph{rounding} a particular moment-based relaxation proposed in \cite{risteski2016calculate}.
\end{comment}`'
More precisely, a recent result due to \cite{jain2018mean} proves that the mean-field approximation to $\F$ is within an additive error\footnote{Here, $\|J\|_F := \sqrt{\sum_{i,j}J_{i,j}^2}$ is the \emph{Frobenius norm} of the matrix $J$.} of $O(n^{2/3} \|J\|^{2/3}_F \log^{1/3}(n \|J\|_F))$. We improve this and show:
\begin{theorem}
\label{thm:new-mean-field}
Fix an Ising model $J$ on $n$ vertices. Then,
$$\F - \F^{*} \leq 3n^{2/3} \|J\|_F^{2/3}.$$
\end{theorem}
%\begin{ntheorem}[Informal] The mean-field approximation to $\F$ is within an additive error of $ O(n^{2/3} \|J\|_F^{2/3})$. 
%\end{ntheorem}
We note that \cite{jain2018mean} prove this inequality is tight up to constants. This also recovers the result of \cite{basak2017universality} which shows the error is $o(n)$ when $\|J\|_F^2 = o(n)$. The full results are in \cref{s:mfbounds}.  \\
\item {\bf Subexponential algorithms for approximating $\F$ up to the computational intractability limit}: Our proof of the above theorem is algorithmic, except that it assumes access to the true Gibbs measure. To fix this, we instead apply our rounding scheme to a convex relaxation proposed by \cite{risteski2016calculate} based on the Sherali-Adams hierarchy. The algorithm we get as a result runs in subexponential time so long as $\|J\|^2_F = o(n)$; this condition for subexponentiality is tight under Gap-ETH. %We furthermore show that assuming Gap-ETH, if $\|J\|^2_F = \Theta(n)$, this is impossible.
More precisely: 
\begin{theorem} We can approximate $\F$ up to an additive factor of $o(n)$ in time $2^{o(n)}$ if $\|J\|^2_F = o(n)$. Moreover, we can also output a product distribution achieving this approximation. On the other hand, for $\|J\|^2_F = \Theta(n)$, it is Gap-ETH-hard to approximate $\F$ up to an additive factor of $o(n)$ in subexponential time. 
\label{thm:algorithm-imprecise}
\end{theorem}
We also describe how to accelerate the algorithm on dense graphs using random subsampling.
The full results are in \cref{s:algorithmic}.
\item {\bf Optimality of correlation rounding}: The rounding we use in the proof of the above theorems relies crucially on the \emph{correlation rounding} technique introduced in \cite{barak2011rounding}. This procedure was designed specifically to tackle dense and spectrally well-behaved instances of constraint satisfaction problems, as well as to derive subexponential algorithms for unique games. In order to better understand the efficacy of correlation rounding, Allen, O'Donnell, and Zhou \cite{allen2015conditioning} introduced a conjecture on the number of variables one needs to condition on in an arbitrary distribution, in order to guarantee that the remaining pairs of variables have average covariance at most $\epsilon$. The current best result of \cite{raghavendra-tan} gives a bound of $O(1/\epsilon^2)$; \cite{allen2015conditioning} conjectured that this can be decreased to $O(1/\epsilon)$. We refute this conjecture in essentially the strongest possible sense. Namely, we show:  
%\begin{theorem} For all $n$, there exists $t_n  = \omega_n(1)$ and jointly distributed $\{\pm 1\}$-valued random variables $X_1,\dots,X_n$ s.t. for any set $T$ with $|T| \leq t_n$, we have 
%$$\E_{(i,j) \sim {[n] \choose 2}} \left[\left|\Cov(X_i,X_j) \right| | (X_k)_{k\in T}\right] = \Omega\left(\frac{1}{\sqrt{t_n}}\right). $$
%\end{theorem}
\begin{theorem}
\label{thm:refutation-AOZ}
There exists an absolute constant $C > 0$, a sequence of pairs $(t_n,n)$ going to infinity, and a family of probability distributions (the SK spin glass) such that for \emph{any set} $T$ with $|T| \le t_n$,
$$\E_{(i,j) \sim {[n] \choose 2}} \left[\left|\Cov(X_i,X_j) \right| | (X_k)_{k\in T}\right] \ge \frac{C}{\sqrt{t_n}}. $$
\end{theorem}

We prove this theorem by combining our techniques with rigorous results on the Sherrington-Kirkpatrick spin glass.
The full results are in \cref{s:odonnell}. 
\item \textbf{Generalization of all results to $k$-MRFs: } We give  natural and tight generalizations of these results to order $k$ Markov Random Fields. In general, we show that the tight regime for $o(n)$ additive error for both mean-field and sub-exponential time algorithms (under Gap-ETH) is $\|J\|_F^2 = o(n^{3 - k})$, and show tightness of the higher-order correlation rounding guarantee. The full results are in \cref{sec:mrf}.
\end{enumerate}

\section{Background and related work}
\subsection{The mean-field approximation}
Owing to its simplicity, the mean field approximation %for Ising models 
has long been used in statistical physics (see \cite{parisi1988statistical} for a textbook treatment) and also in Bayesian statistics \cite{peterson-anderson,jordan1999introduction,wainwright-jordan-variational}, where it is one of the prototypical examples of a \emph{variational method}. It has the attractive property that it always gives a lower bound for the free energy. 

The critical points of $\F^*$ have a fixpoint interpretation as the solutions to the mean-field equation, $x = \tanh^{\otimes n}(J x)$. However, iterating this equation is known to converge to the mean-field solution only in high-temperature regimes such as Dobrushin uniqueness; as soon as we leave this regime, the iteration may fail to converge to the optimum even in simple models (Curie-Weiss) -- see \cite{jain2018mean}. We explain a connection between the mean-field equation and our approach in \cref{s:local-fields} that does not rely on any high-temperature assumption.

It is well known \cite{ellis-newman} that the mean field approximation
is very accurate for the Curie-Weiss model, which is the Ising
model on the complete graph, at all temperatures. On the other hand, it is also known \cite{DemboMontanari:10} that for very sparse graphs like trees of bounded arity, this is not the case. 

In recent years, considerable effort has gone into bounding the error of the mean-field approximation on more general graphs; we refer the reader to \cite{basak2017universality, jain2018mean} for a detailed discussion and comparison of results in this direction. %\cite{basak2017universality}, who provided an exponential improvement over previous work of %Borgs, Chayes, Lov\'asz, S\'os, and Vesztergombi 
%\cite{borgs2012convergent} to show that
If one only wishes to show that the mean-field approximation asymptotically gives the correct free energy density $\F/n$ and does not care about the rate of convergence, then the breakthrough result is due to %Basak and Mukherjee 
\cite{basak2017universality}, who provided an exponential improvement over previous work of %Borgs, Chayes, Lov\'asz, S\'os, and Vesztergombi 
\cite{borgs2012convergent} to identify the regime where this happens.
\begin{theorem}[\cite{basak2017universality}]
\label{thm:bm}
Let $(J_n)_{n=1}^{\infty}$ be a sequence of Ising models indexed by the number of vertices. 
if $\|J_n\|^{2}_{F} = o(n)$, then $\F_{J_n} - \F^*_{J_n} = o(n)$. 
\end{theorem}
This result is tight -- there are simple
examples of models with $\|J_n\|_F^2 = \Theta(n)$ where $\F_{J_n} - \F^*_{J_n} = \Omega(n)$. On the other hand, if one also cares about the rate of convergence, then this result is not the best known. Here, improving on previous bounds of  
%In terms of rate of convergence there are recent developments. Improving on previous bounds of %Borgs et al. 
\cite{borgs2012convergent}, %Basak and Mukherjee 
\cite{basak2017universality}, and %Eldan 
\cite{eldan2016gaussian}, it was shown by %Jain, Koehler and Mossel 
\cite{jain2018mean} that: 
\begin{theorem}[\cite{jain2018mean}]
\label{thm:jkm}
Fix an Ising model $J$ on $n$ vertices. Then,
$$\F - \F^{*} \leq 200 n^{2/3} \|J\|_F^{2/3} \log^{1/3}(n \|J\|_F + e).$$
\end{theorem}
As stated earlier, our first main result \cref{thm:new-mean-field} removes the logarithmic term in \cref{thm:jkm}, thereby completely subsuming both of the theorems stated above.
%\begin{theorem}
%\label{thm:new-mean-field}
%Fix an Ising model $J$ on $n$ vertices. Then,
%$$\F - \F^{*} \leq 3n^{2/3} \|J\|_F^{2/3}.$$
%\end{theorem}
A more general version of this theorem, valid for higher-order Markov random fields on arbitrary finite alphabets, is \cref{thm:MRF-mean-field} below. 

\subsection{Algorithms for dense graphs}

At first glance, the condition that $\|J\|_F^2 = o(n)$ may seem a little odd. 
To demystify it, consider the anti-ferromagnetic Ising model corresponding\footnote{The scaling here is chosen so that if the MAX-CUT is $\gamma n$ edges with $\gamma > 1/2$, then the two terms in \eqref{eqn:gibbs} are of the same scale.}
to MAX-CUT on a graph with $m$ edges
which has $J_{ij} = -\frac{\beta n}{m}$ for each $(i,j) \in E$. If $M$ is the optimum fraction of edges cut, then
\begin{equation}\label{eqn:max-cut-ising}
\frac{1}{n\beta} \log Z \in \left[\text{M} - \frac{1}{\beta}, \text{M} + \frac{1}{\beta}\right], \qquad \|J\|_F^2 = \Theta\left(\beta^2 \frac{n^2}{m}\right),
\end{equation}
so the requirement that $\|J\|_F^2 = o(n)$ is the same as requiring $m = \omega(n)$. In other words, our algorithms operate in the regime where the average degree is super-constant and the objective is to approximate MAX-CUT within factor $(1 - \epsilon)$.
Thus, they can be viewed as free-energy generalizations of optimization problems on dense graphs. 

We briefly survey relevant work on approximation algorithms for dense graphs. The main emphasis in the literature has been on the case when $m = \Theta(n^2)$ for which PTASs have been developed, for instance the weak regularity lemma based algorithm of \cite{frieze-kannan-matrix}, the greedy algorithms of \cite{mathieu-schudy}, and the Sherali-Adams based approach of \cite{delavega2007linear}. On the other hand, if $m = \Theta(n^{2 - \epsilon})$ for any $\epsilon > 0$ then no PTAS for even MAX-CUT is possible \cite{less-dense-maxcut-hardness}. 

%Several polynomial time algorithms are known in the $m = \Theta(n^2)$ case: for example there are methods based on the weak regularity lemma \cite{frieze-kannan-matrix}, greedy algorithms \cite{mathieu-schudy}, and the natural Sherali-Adams relaxation \cite{delavega2007linear}. 

The work most relevant to ours is the improved analysis of the Sherali-Adams relaxation due to \cite{yoshida-zhou} based on correlation rounding. Surprisingly, although there are many methods to approximate MAX-CUT when $m = \Theta(n^2)$ as mentioned above, to our knowledge \emph{none of the algorithms except for Sherali-Adams} are guaranteed to give sub-exponential time algorithms down to $m = \omega(n)$; for example, the method of \cite{frieze-kannan-matrix} is only sub-exponential time for $m = \omega(n \log n)$. The guarantee for Sherali-Adams in this regime is not explicitly stated in \cite{yoshida-zhou} or anywhere else, as far as we are aware, but is straightforward to show even from the correlation rounding guarantee of \cite{raghavendra-tan} (see \cref{s:algorithmic}). The correct generalization of this guarantee for MAX-$k$-CSP was essentially pointed out in \cite{fotakis2015sub} but once again, their algorithm misses the tight regime (achievable by Sherali-Adams) by poly-logarithmic factors. Our result recovers the tight regime (i.e. $\omega(n^{k - 1})$ constraints) in this setting as well, while also generalizing to the free-energy (see \cref{s:algorithmic}).

For computing the free energy, the two most relevant works are \cite{risteski2016calculate} and \cite{jain2018mean}: the first work does not make any connection to mean-field approximation and proves a slightly weaker guarantee for Sherali-Adams than the current work; the second work uses a regularity based approach to compute the mean-field approximation, and gets similar guarantees to the algorithm of this work but misses the correct sub-exponential time regime by log factors.

\subsection{Correlation rounding, and a refutation of the Allen-O'Donnell-Zhou conjecture}

Let $X_1,\dots,X_n$ be a collection of jointly distributed random variables, each of which takes values in $\{\pm 1\}$. There are two possibilities for such a collection: 
\begin{itemize}
\item The average covariance of the collection, defined to be $\E_{(i,j)\sim {[n] \choose 2}}|\Cov(X_i,X_j)|$, is small.
\item The average covariance of the collection is not small: in this case, we expect a random coordinate $X_j$ to contain significant information about many of the other random variables in $X_1,\dots,X_n$, so that we might intuitively conjecture that conditioning on the random variables $X_j$ for all $j$ in a `small' random subset $T$ of $[n]$ makes the average covariance sufficiently small.
\end{itemize}

This intuition is indeed true, and has been quantitatively formalized in several works by the theoretical computer science community \cite{barak2011rounding,guruswami2011lasserre,raghavendra-tan,yoshida-zhou}. We note that similar ideas have appeared independently in the statistical physics literature under the name of `pinning'; see e.g. \cite{ioffe2000note} and references therein, as well as in the recent work \cite{coja2017bethe}. 
\begin{theorem}[\cite{raghavendra-tan}]
\label{thm:covariance-version-corr-rounding}
Let $X_1,\dots,X_n$ be a collection of $\{\pm 1\}$-valued random variables, and let $0<\epsilon \leq 1$. Then, for some integer $0\leq t\leq O(1/\epsilon^{2})$:
$$\E_{T \sim {V \choose t}} \E_{(i,j) \sim {[n] \choose 2}} \left[\left|\Cov(X_i,X_j) \right| | (X_k)_{k\in T}\right] \leq \epsilon.$$
\end{theorem}

The above theorem is at the heart of the so-called \emph{correlation rounding} technique for the Sherali-Adams and SOS convex relaxation hierarchies, which has been used to provide state-of-the-art approximation algorithms for many classic NP-hard problems and their variants; we refer the reader to the references above for much more on this. As we will see below, it will also be key to our proof of \cref{thm:new-mean-field}.\\ 

Recently, it was conjectured by Allen, O'Donnell and Zhou \cite{allen2015conditioning} that the upper bound on $t$ in \cref{thm:covariance-version-corr-rounding} can be improved significantly. More precisely, they conjectured that:

\begin{conjecture}[Conjecture A in \cite{allen2015conditioning}]
\label{conjecture:AOZ}
\cref{thm:covariance-version-corr-rounding} holds with $0\leq t\leq O(1/\epsilon)$.
\end{conjecture}
Their motivation for this conjecture was twofold: 
\begin{itemize}
\item On a technical level, the proof of \cref{thm:covariance-version-corr-rounding} in \cite{raghavendra-tan} proceeds by first showing that for some integer $0\leq t \leq O(1/\epsilon^{2})$
$$\E_{T \sim {V \choose t}} \E_{(i,j) \sim {[n] \choose 2}} \left[\left|I(X_i,X_j) \right| | (X_k)_{k\in T}\right] \leq \epsilon^{2},$$
where $I(X,Y)$ denotes the \emph{mutual information} between $X$ and $Y$, and then using the standard inequality $|\Cov(X,Y)| \leq \sqrt{2I(X,Y)}$; we will present a generalized version of this proof from \cite{manurangsi2017birthday,yoshida-zhou} later. Essentially, they conjectured that one could surmount the quadratic loss %in switching from the mutual information to the covariance %,possibly by proving a statement like \cite{raghavendra-tan} directly for covariance 
without passing through mutual information.
\item From a complexity-theoretic point of view, the best lower bounds on
the computational complexity of dense MAX-CSP problems (such as \cite{ailon2007hardness,manurangsi2017birthday}) leave open the possibility that MAX-CUT on $n$ vertices can be computed to within $\epsilon n^2$ additive error in time $n^{O(1/\epsilon)}$, whereas the best known algorithms all require time at least $2^{O(1/\epsilon^2)}$. If \cref{conjecture:AOZ} were true, the running time of the Sherali-Adams based approach would have improved to $n^{O(1/\epsilon)}$ time for $\epsilon n^2 \|J\|_{\infty}$ error (which, for dense graphs, is close to matching the lower bound of \cite{manurangsi2017birthday}).
\end{itemize}

\cite{allen2015conditioning} prove \cref{conjecture:AOZ} for the special case when the random variables $X_1,\dots,X_n$ are the leaves of a certain type of information flow tree known as the caterpillar graph. In addition, \cite{manurangsi2017birthday} showed a similar improvement for correlation rounding in the MAX $k$-CSP problem, when promised that there exists an assignment satisfying all of the constraints. As described in the introduction, we use ideas from statistical physics to refute \cref{conjecture:AOZ} in essentially the strongest possible form by showing that \cref{thm:covariance-version-corr-rounding} does not hold with $0\leq t\leq o(1/\epsilon^2)$ (\cref{thm:refutation-AOZ}). 

\section{Technical tools}
\subsection{Hierarchies of convex relaxations} %relaxations and the Sherali-Adams hierarchy}
Computing the free energy of an Ising model has as a special case the problem
MAX-QP/MAX-2CSP, because if we let $J_{\beta} = \beta J$ then
\begin{equation}
\lim_{\beta \to \infty} \frac{1}{\beta} \log Z(J_{\beta}) = \lim_{\beta \to \infty} \sup_{\mu} \left(\frac{1}{2} \E[X^T J X]+ \frac{1}{\beta} H(\mu)\right) = \max_{x \in \{\pm 1\}^n} x^T J x.
\end{equation}
%\emph{Convex relaxations} are one of the most important tools in the study of combinatorial optimization problems. 
As with many other problems in combinatorial optimization, this is a  maximization problems on the Boolean hypercube, i.e. as a problem of the form
\[ \max_{x \in \{\pm 1\}^n} f(x). \]
These problems are often NP-hard to solve exactly, but \emph{convex hierarchies} give a principled way to write down a natural family of convex relaxations which are efficiently solvable and give increasingly better approximations to the true value.
First, one re-expresses the problem as an optimization problem over the convex polytope of \emph{probability distributions} using that
\[ \max_{x \in \{\pm 1\}^n} f(x) = \max_{\mu \in \mathcal{P}(\{ \pm 1\}^n)} \E_{\mu}[f(x)]; \]
the advantage of this reformulation is that the objective is now linear in the variable $\mu$. Second, one relaxes $\mathcal{P}(\{ \pm 1\}^n)$ to a larger convex set of \emph{pseudo-distributions} which are more tractable to optimize over. The tightness of relaxation is controlled by a parameter $r$ (known as the \emph{level} or \emph{number of rounds} of the hierarchy); as the parameter $r$ increases, the relaxation becomes tighter with the level $n$ relaxation corresponding to the original optimization problem. 

Different hierarchies correspond to different choices of the space of pseudo-distributions; two of the most popular are the \emph{Sherali-Adams (SA) hierarchy} and the \emph{Sum-of-Squares (SOS)/Laserre hierarchy}. %\emph{Sherali-Adams hierarchy}, which can be efficiently optimized over by linear programming, and the stronger \emph{Sum-of-Squares (SOS)/Laserre hierarchy} which enforces additional constraints via the use of semidefinite programming. 
In the \emph{Sherali-Adams hierarchy}, we define a \emph{level $r$-pseudodistribution} to be given by the following variables and constraints:
\begin{enumerate}
\item For every $S \subset [n]$ with $|S| = r$, a valid joint distribution $\mu_S$ over $\{\pm 1\}^S$.
\item \emph{Compatability conditions}, which require that for every $U \subseteq [n]$ with $|U| \le r$ and every $S,S'\subseteq [n]$ with $|S|=|S'|=r$ and $U \subset S \cap S'$, $\mu_S|_U = \mu_{S'}|_U$. 
\end{enumerate}
Observe that, by linearity, this data defines a unique \emph{pseudo-expectation operator}\footnote{This operator may behave very differently from a true expectation. For example, it's possible that $\tilde{\E}[f^2] < 0$ for some $f$. The SOS hierarchy is formed by additionally requiring $\tilde{\E}[f^2] \ge 0$ for all low-degree $f$.} $\tilde{\E}$ from real polynomials of degree at most $r$ to $\mathbb{R}$. 

Let $SA_r$ denote the set of level $r$-pseudodistributions on the hypercube.
Then for $r \ge \deg(f)$, we can write down $\max_{\mu \in SA_r} \tilde{\E}_{\mu}[f(x)]$ as a linear program with $2^r {n \choose r}$ many variables and a number of constraints which is polynomial in the number of variables. 
%In particular, by standard LP algorithms, the optimum for the $r^{th}$ level of the $SA$-hierarchy can be computed in time $\poly\left(2^{r}\binom{n}{r}\right)$. 
By \emph{strong duality} for linear programs, we can also think of the value of the level $r$ $SA$ relaxation as corresponding to the best upper bound derivable on $\sup_{\mu} \E_{\mu}[f(x)]$ in a limited proof system, which captures e.g. case analysis on sets of size at most $r$.

In addition to this standard setup, since the variational formulation for $\log Z$ has an entropy term, we will need a proxy for it when we use the Sherali-Adams hierarchy. The particular proxy we will use was introduced by \cite{risteski2016calculate} -- further details are in Section~\ref{s:algorithmic}.

\subsection{The correlation rounding theorem}
As mentioned in the introduction, our proof of \cref{thm:new-mean-field} will depend crucially on the correlation rounding theorem. Here, we present a general higher-order version of this theorem due to \cite{manurangsi2017birthday}, building on previous work of \cite{raghavendra-tan} and \cite{yoshida-zhou}. %We start by introducing some information theoretic notions. 
\begin{comment}
\begin{definition}
The \emph{multivariate mutual information} of a collection of random variables $X_1,\dots,X_n$ is defined to be 
\[ I(X_1; \cdots; X_n) = \sum_{m = 1}^n (-1)^{m-1} \sum_{S \subset {n \choose m}} H(X_S). \]
\end{definition}
Note that when $n=2$, this corresponds to the usual notion of mutual information between two random variables. We may also define the \emph{conditional multivariate mutual information} by using the conditional entropy in the above equation; note that the chain rule for entropy shows immediately that 
\[ I(X_1; \cdots; X_n) = I(X_1; \cdots; X_{n - 1}) - I(X_1; \cdots; X_{n - 1} | X_n). \]
\end{comment}
\begin{definition}
The \emph{multivariate total correlation} of a collection of random variables $X_1,\dots,X_n$ is defined to be
\[ C(X_1; \cdots; X_n) = \KL\left(\mu(X_{1,\ldots,n}) || \mu(X_1) \times \cdots \times \mu(X_n)\right).\]
\end{definition}
From the definition of $\KL$ divergence, it follows that 
\[ C(X_1; \cdots; X_n) = \left(\sum_{i=1}^{n}H(X_i)\right) - H(X_1,\ldots,X_n).\]
By using conditional distributions/ conditional entropies, we may define the \emph{conditional multivariate total correlation} in the obvious way. Note that in the two-variable case, the total correlation is the same as the \emph{mutual information} $I(X_1;X_2)$.
\begin{theorem}[Correlation rounding theorem, \cite{manurangsi2017birthday}]
\label{thm:correlation-rounding}
Let $X_1,\dots,X_n$ be a collection of $\{\pm 1\}$-valued random variables. Then, for any $k,\ell \in [n]$, there exists some $t\leq \ell$ such that:  
\[ \E_{S \sim \binom{V}{t}} \E_{F \sim {V \choose k}}[C(X_F | X_S)] \le \frac{k^2 \log(2)}{\ell}.\]
\end{theorem}
\begin{remark}
The same conclusion holds for general random variables $X_1,\dots,X_n$ with the factor $\log{2}$ replaced by $\frac{\sum_{i=1}^{n}H(X_i)}{n}$. Also, the guarantee holds for general level $(\ell + k)$-pseudodistributions.
\end{remark}
For the reader's convenience, we provide a complete proof of this result in \cref{appendix:proof-of-corr-rounding}, correcting certain errors which have been persistent in the literature.

\subsection{The Sherrington-Kirkpatrick model and spin glass theory}
The Sherrington-Kirkpatrick (SK) spin glass model was introduced in \cite{kirkpatrick1975solvable} as a solvable model of disordered systems. The Gibbs measure of the SK spin glass on $n$ vertices (without external field) is a random probability distribution on $\{\pm 1\}^{n}$ given by:
\[ \Pr(X = x) := \frac{1}{Z_{n}(\beta)}\exp\left(\frac{\beta}{\sqrt{n}}\sum_{1 \le i < j \le n} J_{ij} X_i X_j\right),\]
where $J_{ij} \sim N(0,1)$ are i.i.d. standard Gaussians and $\beta$ is a fixed parameter referred to as the \emph{inverse temperature}.
In \cite{kirkpatrick1975solvable}, a prediction, now known as the \emph{replica-symmetric prediction}, was made for the limiting value of $\frac{1}{n} \log Z_n(\beta)$ as $n \to \infty$.
It was soon realized that this prediction could not be correct for all values of $\beta$; finding and understanding the correct prediction led physicists to the development of a sophisticated \emph{spin glass theory} based upon the non-rigorous \emph{replica method} (\cite{mezard1987spin}). In particular, physicists showed via the replica method that the SK spin glass exhibits two \emph{phases} depending on the value of $\beta$:
\begin{enumerate}
\item \emph{Replica Symmetry} (RS, $\beta < 1$). This is the regime where the original prediction for the limiting value of $\frac{1}{n} \log Z_{n}(\beta)$ is correct. Moreover, the Gibbs measure exhibits a number of unusual properties: for example, the marginal law on any small subset of the coordinates converges to a product distribution as $n \to \infty$ (\cite{Talagrand:11}). %In the presence of external field the marginals are approximate fixpoints of the famous TAP equations \cite{Talagrand:11}. 
\item \emph{(Full) Replica Symmetry Breaking} (fRSB, $\beta > 1$). In this phase, the limit of $\frac{1}{n} \log Z_n(\beta)$ does not have a simple closed form; however, there is a remarkable variational expression for the limiting value known as the \emph{Parisi formula}. Moreover, the Gibbs measure is understood to be shattered into exponentially many clusters with the geometry of an \emph{ultrametric space}. %More precisely, when there is no external field the SK spin glass is understood to exhibit what is known as \emph{Full Replica Symmetric Breaking} (fRSB). 
\end{enumerate} 
In the replica symmetric phase, the prediction for the limiting value of $\frac{1}{n} \log Z_{n}(\beta)$ was rigorously confirmed by the work of \cite{aizenman1987}. %(see also \cite{talagrand2011mean}). 
Furthermore, they proved their result for general distributions of the $J_{ij}$, giving what is known as a \emph{universality} result.
\begin{theorem}[\cite{aizenman1987}]\label{thm:rs-correct}
Let $\epsilon > 0$. For the SK spin glass at inverse temperature $\beta < 1$,
\[ \Pr\left(\left|\frac{1}{n} \log Z_{n}(\beta) - (\log 2 + \beta^2/4)\right| \ge \epsilon\right) \to 0 \]
as $n \to \infty$. Moreover, this also holds if the $J_{ij}$ are i.i.d. samples from \emph{any} distribution with finite moments, mean $0$ and variance $1$. %sampled i.i.d. from any distribution with mean 0, variance $\beta^2/n$, and such that all moments exist.
\end{theorem}
This is the only result we will need from the spin glass literature, although much more is now rigorously known. 
For an account of more recent developments, including the proofs of the Parisi formula and ultrametricity conjecture, we refer the reader to the books \cite{panchenko2013sherrington,Talagrand:11,talagrand2011mean}.
%It was first rigorously shown by Toninelli that this is the sharp threshold; the replica symmetric prediction is incorrect for $\beta > 1$ \cite{?}. 
\section{Mean-field approximation via correlation rounding: proof of \cref{thm:new-mean-field}}
\label{s:mfbounds}
First we recall a couple of lemmas which are essentially used in all works on correlation rounding. Recall that for two probability distributions $P$ and $Q$ on the same finite space $\Omega$, the total variation distance between $P$ and $Q$ is defined by
$\TV(P,Q):= \sup_{A\subseteq \Omega}\left|\sum_{a\in A} \left(P(a)-Q(a)\right)\right|$.
%First, we recall the following well-known fact (e.g. as Lemma 5.1 of \cite{barak2011rounding}):
\begin{lemma}[Lemma 5.1, \cite{barak2011rounding}]\label{lem:tv-is-cov}
Let $X$ and $Y$ be jointly distributed random variables valued in $\{\pm 1\}$. Let $P_X,P_Y$ denote the marginal distributions of $X$ and $Y$, and let $P_{X,Y}$ denote their
joint distribution. Then,
\[ |\Cov(X,Y)| = 2\TV(P_{X,Y},P_X \times P_Y). \]
\end{lemma}
\begin{comment} % this proof is fairly unenlightening
\begin{proof}
Since $\E[X]=P_{X}(1)-P_{X}(-1)$, $\E[Y]=P_{Y}(1)-P_{Y}(-1)$ and
$\E[XY]=P_{X,Y}(1,1)+P_{X,Y}(-1,-1)-P_{X,Y}(1,-1)-P_{X,Y}(-1,1)$,
it follows by definition that
\begin{align*}
|\Cov(X,Y)| & =\left|\E[XY]-\E[X]\E[Y]\right|\\
 & =\left|\sum_{(i,j)\in\{\pm 1\}^{2},i=j}\left(P_{X,Y}(i,j)-P_{X}(i)P_{Y}(j)\right)-\sum_{(i,j)\in\{\pm 1\}^{2},i\neq j}\left(P_{X,Y}(i,j)-P_{X}(i)P_{Y}(j)\right)\right|\\
 & =\left|2\sum_{i\in\{\pm1\}}\left(P_{X,Y}(i,i)-P_{X}(i)P_{Y}(i)\right)\right|\\
 & \leq2\TV\left(P_{X,Y},P_{X}\times P_{Y}\right),
\end{align*}
where the last inequality follows by the definition of $\TV$ above. 
%Without loss of generality suppose $\Cov(X,Y) \ge 0$.
%Recall that in general, $\TV(P,Q) = \sum_a (p_a - q_a) \bone(p_a \ge q_a)$. In our case, this gives
%\begin{align*}
%\TV(P_X \times P_Y, P_{X,Y}) 
%&= (\Pr(X = Y = 1) - \Pr(X = 1)\Pr(Y = 1)) \\
%&\quad + (\Pr(X = Y = -1) - \Pr(X = -1)\Pr(X = -1))
%\end{align*}
%and we see the right hand side is $\Cov(X,Y)$. 
\end{proof}
In fact, it is also easily seen that in our setting, $|\Cov(X,Y)| = 2\TV\left(P_{X,Y},P_X \times P_Y\right)$, although we will not need this.  
\end{comment}
From this, one can observe the following consequence of correlation rounding:%; this is standard and we include the proof for completeness.
\begin{lemma}
\label{lemma:covariance-corr-rounding}
Let $X_1,\dots, X_n$ be a collection of $\{\pm 1\}$-valued random variables. Then, for any $\ell \in [n]$, there exists some $S \subset [n]$ with $|S| \leq \ell$ such that:
$$\E_{X_S} \E_{\{u,v\} \in {V \choose 2}}\left[\Cov(X_u,X_v | X_S)^2\right] \le \frac{8\log 2}{\ell}.$$
\end{lemma}
\begin{proof}
This is standard and we include the proof for completeness.
We begin by applying \cref{thm:correlation-rounding} with $\ell$; let $S$ denote the resulting set of size at most $\ell$. 
By Pinsker's inequality, we have
$$2\TV^2\left(\mu(X_{u,v} | X_S=x_s), \left(\mu(X_u | X_S=x_s) \times \mu(X_v | X_S=x_s)\right)\right) \le C(X_u;X_v | X_S=x_s),$$
for any $x_s \in \{\pm 1\}^{|S|}$. 
Therefore, by taking the expectation on both sides, we get:
\[ 2\E_{X_S} \TV^2\left(\mu(X_{u,v} | X_S), \left(\mu(X_u | X_S) \times \mu(X_v | X_S)\right)\right) \le C(X_u;X_v | X_S). \] 
By averaging over the choice of $\{u,v\} \in \binom{V}{2}$, we get
\begin{align*}
\E_{E=\{u,v\} \sim {V \choose 2}}\E_{X_S}\left[\TV^2\left(\mu(X_{u,v} | X_S), \left(\mu(X_u | X_S) \times \mu(X_v | X_S)\right)\right)\right] &\le \E_{E \sim \binom{V}{2}}\left[\frac{C(X_E|X_S)}{2}\right]\\
 & \leq \frac{2\log 2}{\ell},
\end{align*}
where the second inequality follows by the choice of $S$ and \cref{thm:correlation-rounding}. 
Finally, \cref{lem:tv-is-cov} shows that for any $x_s \in \{\pm 1\}^{|S|}$, 
$$|\Cov(X_u,X_v | X_S = x_S)| \leq 2\TV\left(\mu(X_{u,v} | X_S = x_S), \left(\mu(X_u | X_S = x_S) \times \mu(X_v | X_S = x_S)\right)\right),$$ from which we obtain our desired conclusion: 
\begin{equation}\label{eqn:cov-sq-bound}
\E_{X_S} \E_{\{u,v\} \in {V \choose 2}}\left[\Cov(X_u,X_v | X_S)^2\right] \le \frac{8\log 2}{\ell}.
\end{equation}
\end{proof}
Finally, we recall the \emph{maximum-entropy principle} characterizing product distributions:
\begin{lemma}
\label{lemma:entropy-first-moments}
Let $\mu$ denote a probability distribution on the finite space $\Omega_{1}\times \dots \times \Omega_{n}$. Let $\nu$ denote the product distribution on $\Omega_{1}\times \dots \times \Omega_{n}$ whose marginal distribution on $\Omega_{i}$ is the same as that of $\mu$ for all $i \in [n]$. Then, $H(\mu) \leq H(\nu)$.  
\end{lemma}
\begin{proof}
This is a direct application of the chain rule and tensorization for entropy. Indeed, let $X := (X_1,\dots,X_n) \sim \mu$. Then,
$$H(\mu) = H(X) \leq H(X_1)+\dots+H(X_n) = H(\nu).$$
\end{proof}
%\begin{theorem}
%\[ 0 \le \F - \F^* \le 3 n^{2/3} \|J\|_F^{2/3} \]
%\end{theorem}
%We are now ready to prove our first main result.
%\begin{theorem}
%\label{thm:new-mean-field}
%Fix an Ising model $J$ on $n$ vertices. Then,
%$$\F - \F^{*} \leq 3n^{2/3} \|J\|_F^{2/3}.$$
%\end{theorem}
We are now ready to prove \cref{thm:new-mean-field}. 
\begin{proof}[Proof of \cref{thm:new-mean-field}]
Let $\epsilon > 0$ be some parameter which will be optimized later. We begin by applying \cref{lemma:covariance-corr-rounding} with $\ell = 1/(\epsilon^{2}\log{2})$ (for clarity of exposition, we will omit floors and ceilings since they do not make any essential difference); let $S$ denote the resulting set of size at most $\ell$. Let $\mu$ denote the Boltzmann distribution, and recall that the Gibbs variational principle \cref{eqn:gibbs} states that
$$\F = \E_{\mu}\left[X^{T}JX\right] + H_{\mu}(X).$$
%Recall that the Gibbs variational principle \cref{eqn:free-energy-variational-char} states that:
%\[ \log Z = \sup_{\mu} \left[\sum_{ij} \E_{\mu}[ X^T J X ] + H_{\mu}(X)\right] \]
%and equality is achieved precisely when $\mu$ is the Boltzmann distribution.
Let $\nu_{x_S}$ denote the product distribution on $\{\pm 1\}^{n}$ for which $\E_{\nu_{x_S}}[X_i] = \E[X_i | X_S = x_S]$.
Then, using the chain rule for entropy, we see that
\begin{align*}
\F
&= \sum_{i < j} J_{i,j}\E_{\mu}[X_iX_j] + H_{\mu}(X) \\
&= \sum_{i < j} J_{i,j}\E_{\mu}[X_iX_j] + H_{\mu}(X|X_S) + H_{\mu}(X_S)\\
&= \E_{x_S}\left[\sum_{i<j} J_{ij}\E_{\mu}[X_i X_j | X_S = x_S] + H_{\mu}(X | X_{S} = x_S)\right] + H_{\mu}(X_S) \\
&\le \E_{x_S}\left[\sum_{i<j} J_{ij}\E_{\mu}[X_i X_j | X_S = x_S] + H_{\mu}(X | X_{S} = x_S)\right] + 1/\epsilon^2 \\
& \leq \E_{x_S}\left[\sum_{i<j} J_{ij}\E_{\mu}[X_i X_j | X_S = x_S]] + H_{\nu_{x_s}}(X)\right] + 1/\epsilon^2,
\end{align*}
where in the fourth line, we have used that $|S| \le \ell = 1/(\epsilon^{2}\log{2})$, and in the last line, we have used \cref{lemma:entropy-first-moments}. 
From \cref{lemma:covariance-corr-rounding} and the Cauchy-Schwarz inequality, it follows that
\begin{align*}
\E_{X_S}\left[\sum_{i<j} J_{ij}\E_{\mu}[X_i X_j | X_S] \right]
&= \E_{x_S}\left[\sum_{i<j} J_{ij}\left(\Cov(X_i,X_j|X_S=x_S)+\E_{\mu}[X_i|X_S = x_S] [X_j | X_S = x_S]\right)\right]\\
&= \sum_{i<j}J_{i,j}\E_{X_S}[\Cov(X_i,X_j|X_S)] + \E_{X_S} \sum_{i,j}J_{i,j}\E_{\nu_{X_S}}[X_iX_j]\\
&\leq \sqrt{\sum_{i<j}J_{i,j}^{2}}\sqrt{2\binom{|V|}{2}\E_{X_S} \E_{E \in {V \choose 2}}\left[\Cov(X_u,X_v | X_S)^2\right]} + \E_{X_S} \sum_{i<j}J_{i,j}\E_{\nu_{X_S}}[X_iX_j]\\
&\leq 2\epsilon n \|J\|_{F} + \E_{X_S}\sum_{i<j}J_{i,j}\E_{\nu_{X_S}}[X_iX_j]. 
\end{align*}
To summarize, we have shown that
$$\F \leq \E_{x_S}\left[\sum_{i<j}J_{i,j}\E_{\nu_{x_S}}[X_iX_j] + H_{\nu_{x_S}}(X) \right] + 2\epsilon n \|J\|_{F} + \frac{1}{\epsilon^{2}}.$$ 
%Therefore from \cref{eqn:cov-sq-bound} and Cauchy-Schwarz we find
%\begin{align*}
%\log Z \le E_{x_S}[\sum_{ij} J_{ij}\E_{\nu_{x_S}}[X_i X_j] + H_{\nu_{x_S}}(X)] + 5 \epsilon n \|J\|_F + 1/\epsilon^2
%\end{align*}
In particular, there exists some choice of $x_S$, such that with $\nu:= \nu_{x_S}$, we have
\[\F \le \left[\sum_{i<j} J_{ij}\E_{\nu}[X_i X_j] + H_{\nu}(X)\right] + 2\epsilon n \|J\|_F + 1/\epsilon^2.\]
Finally, by setting $\epsilon = \frac{1}{n^{1/3} \|J\|_F^{1/3}}$ we get the desired conclusion:
\[\F \le E_{\nu}\left[\sum_{i<j} J_{ij}X_i X_j + H(X)\right] + 3 n^{2/3} \|J\|_F^{2/3}.\]
\end{proof}
\begin{remark}
For the choice of $\epsilon$ in the above proof to make sense, we require that $\ell = 1/(\epsilon^{2}\log{2}) \leq n$, which translates to $\|J\|_{F}^{2/3} \leq n^{1/3}\log{2}$. However, the above bound also holds if $\|J\|_{F}^{2/3} > n^{1/3}\log{2}$ since in this case, our error term equals $3\log{2}n > 2n$, whereas there is a trivial upper bound of $n\log{2}$ on $\F - \F^*$, obtained by considering the product distribution supported at the point $\arg\max_{x\in \{\pm 1\}^{n}}\{\sum_{ij}{J_{ij}x_ix_j}\}$. 
\end{remark}
\subsection{Aside: correlation rounding and the mean-field equation}
\label{s:local-fields}
The above proof shows that for the product measure $\nu:=\nu_{x_S}$, $\F_\nu$ is close to $\F$. This shows indirectly, by considering the maximizer of $\F^*$, that there exists a product distribution with marginals that are an \emph{exact solution} to the mean-field equation $x = \tanh^{\otimes n}(J x)$ which is close to the Gibbs distribution in $\KL$ distance.
In this subsection, we show that the marginals output by correlation rounding are already an \emph{approximate solution} to the mean-field equation, given slightly stronger assumptions on $J$. This will follow by showing that the variance of the \emph{local fields} $J_i \cdot X$ is greatly reduced by conditioning. We will not need this result henceforth, but this structural result may be of independent interest.
%Since the marginals of the true mean-field solution i.e. of the product measure attaining $\F^*$ satisfy the mean-field equation \cref{}, it is natural to ask whether the marginals of $\nu$ also approximately satisfy this equation. We show that this is indeed true under slightly stronger assumptions about $J$. More precisely:

First, we show that the error in the mean-field equations is bounded by the variance of the local field:
\begin{lemma}\label{lem:mean-field-local-field}
Let $X_1, \ldots, X_n$ be the spins of an Ising model with interaction matrix $J$. Then for any $i$,
\[ |\E[X_i] - \tanh(J_i \cdot \E[X])| \le \sqrt{\Var(J_i \cdot X)} \]
\end{lemma}
\begin{proof}
Since $\E[X_i | X_{\sim i}] = \tanh(J_i \cdot X)$, we know that  $\E[X_i] = \E[\E[X_i | X_{\sim i}]] = \E[\tanh(J_i \cdot X)]$. Therefore,
\begin{align*}
\left|\E[X_i] - \tanh(J_i \cdot \E[X])\right| &\le \E\left|\tanh(J_i \cdot X) - \tanh(J_i \cdot \E[X])\right| \\ 
&\le \E \left|J_i \cdot X - J_i \cdot \E[X]\right| \le \sqrt{\Var(J_i \cdot X)}
\end{align*}
by the triangle inequality, the Lipschitz property of $\tanh$, and Jensen's inequality. 
\end{proof}
Second, we can bound the average variance of the local fields by the average covariance. Recall that the \emph{Schatten 4-norm} of $J$ is given by $\|J\|_{s_4} = \sqrt{\|J^t J\|_F}$.
\begin{lemma}\label{lem:local-field-bound}
Let $X_1, \ldots, X_n$ be arbitrary random variables, and suppose $J_i$ are the rows of a symmetric matrix $J$ with zeros on the diagonal. Then
\[ \frac{1}{n} \sum_{i = 1}^n \Var(J_i \cdot X) \le \|J\|_{s_4}^2 \sqrt{\E_{j,k} \Cov(X_j,X_k)^2} \]
\end{lemma}
\begin{proof}
By expanding out the variance and applying the Cauchy-Schwarz inequality, we find that
\begin{align*}
\frac{1}{n} \sum_{i = 1}^n \Var(J_i \cdot X) 
= \frac{1}{n} \sum_{i,j,k} J_{ij} J_{ik} \Cov(X_j, X_k)
&= \frac{1}{n} \sum_{j,k} J_j \cdot J_k \Cov(X_j,X_k) \\
&\le \|J^t J\|_F \sqrt{\E_{j,k} \Cov(X_j,X_k)^2}.
\end{align*}
Recalling the definition of the Schatten 4-norm gives the result.
\end{proof}
Finally, correlation rounding controls the average covariance, giving us our desired result -- after conditioning, the marginals approximately satisfy the mean-field equation.
\begin{theorem}
Let $X_1,\ldots,X_n$ be the spins of an Ising model with interaction matrix $J$. Fix $\ell$ and let $S$ be the set given by \cref{lemma:covariance-corr-rounding}. Let $Y_i = \E[X_i | X_S]$. Then
\[ \E_{X_S}\left[\frac{1}{n} \sum_{i = 1}^n |Y_i - \tanh(J_i \cdot Y)|^2 \right] = O\left(\frac{1}{\sqrt{\ell}} \|J\|_{s_4}^2\right)\]
\end{theorem}
\begin{proof}
First, we can bound the error in the mean-field equations conditioned on $X_S = x_S$ (which is still an Ising model) by combining \cref{lem:mean-field-local-field} and \cref{lem:local-field-bound}:
\[ \frac{1}{n} \sum_{i = 1}^n |Y_i - \tanh(J_i \cdot Y)|^2 \le \frac{1}{n} \sum_{i = 1}^n \Var(J_i \cdot X | X_S = x_S) \le \|J\|_{s_4}^2 \sqrt{\E_{j,k} \Cov(X_j,X_k | X_S = x_S)^2} \]
Finally, taking the expectation over $X_S$ and applying Jensen's inequality lets us we bound the last term by \cref{lemma:covariance-corr-rounding}.
\end{proof}
The same proof shows the mean-field equation $Y = \tanh^{\otimes n}(J Y + h)$ holds approximately in the presence of external field, after conditioning. 
\begin{comment}
\begin{remark}[Behavior of local fields]\label{rmk:local-fields}
Although we will not need this fact, it is natural to ask
if conditioning on $X_S = x_S$ also significantly reduces
the variance of the \emph{local fields} $J_i \cdot X := \sum_{j}J_{ij}X_j$. This can be used to
show that the resulting marginals approximately satisfy the \emph{mean-field equations}, since in any Ising model
\begin{align*}
\left|\E[X_i] - \tanh(J_i \cdot \E[X])\right| &\le \E\left|\tanh(J_i \cdot X) - \tanh(J_i \cdot \E[X])\right| \\ 
&\le \E \left|J_i \cdot X - J_i \cdot \E[X]\right| \le \sqrt{\Var(J_i \cdot X)}
\end{align*}
by the Lipschitz property of $\tanh$ and Jensen's inequality. 
Under a stronger hypothesis on $J$, we can indeed control the variance of the local fields: 
\[ \frac{1}{n} \sum_{i = 1}^n \Var(J_i \cdot X) = \frac{1}{n} \sum_{j,k} J_j \cdot J_k \Cov(X_j,X_k) \le \|J\|_{s_4}^2 \sqrt{\E_{j,k} \Cov(X_j,X_k)^2} \]
by the Cauchy-Schwarz inequality, where $\|J\|_{s_4} = \sqrt{\|J^t J\|_F}$ is the Schatten 4-norm. Since conditioning on $X_S = x_S$ still gives an Ising model, we can bound the discrepancy from the mean-field equations for $y_i = \E[X_i | X_S = x_S]$:
\[ \frac{1}{n} \sum_{i = 1}^n |y_i - \tanh(J_i \cdot y)|^2 \le \frac{1}{n} \sum_{i = 1}^n \Var(J_i \cdot X | X_S = x_S) = O\left(\frac{1}{\sqrt{\ell}} \|J\|_{s_4}^2\right) \]
where $(S,x_S)$ is chosen by $\ell$ steps of correlation rounding and we applied \cref{lemma:covariance-corr-rounding} to.
\end{remark}
\end{comment}
\section{Correlation rounding is tight for spin glasses: proof of \cref{thm:refutation-AOZ}}
\label{s:odonnell}
We define the following universal constant, which we already know an upper bound on by \cref{thm:covariance-version-corr-rounding}:
\[ \kappa_* := \limsup_{t \to \infty} \sup_{\substack{\mu \in \mathcal{P}(\{\pm 1\}^n) \\ n \ge t}} \min_{S : |S| \le t} \sqrt{t}\ \E_{(i,j) \sim {[n] \choose 2}}[|\Cov(X_i,X_j | X_S)|]. \]
If \cref{conjecture:AOZ} were true, then we would have $\kappa_* = 0$ -- indeed, the conjecture says that
the expected conditional covariance decays like $O(1/t)$, even for a random choice of the conditioning set $S$. We will instead show an explicit positive lower bound on $\kappa_*$, thereby disproving the conjecture. 

We begin by proving a variant of \cref{thm:new-mean-field}, which gives a bound on the error of the mean-field approximation in terms of the constant $\kappa_{*}$. 
\begin{lemma}
\label{lemma:infinity-bound-mean-field}
Let $\{J_n\}_{n\geq 1}$ be a sequence of Ising models indexed by the number of vertices. Let $\F_n$ (resp. $\F^*_n$) denote the free energy (resp. variational free energy) of $J_n$. Suppose that $\kappa_*^{2}\limsup_{n \to \infty}{n\|J_n\|_{\infty}^{2}} < 16$. Then,  
$$\limsup_{n \to \infty}\frac{\F_n - \F_n^*}{n^{4/3}\|J_n\|_{\infty}^{2/3}} \le \frac{3}{\sqrt[3]{4}}\kappa_*^{2/3}.$$
\end{lemma}
\begin{proof}
Let $\{t_n\}_{n\geq 1}$ be a sequence of natural numbers going to infinity, which will be specified later; our choice will be such that $t_n \leq n$ for all $n$. For the Ising model $J_n$, let
$$S_n := \argmin_{S\subseteq[n],|S|\leq t_n}\sqrt{t_n}\ \E_{(i,j) \sim {[n] \choose 2}}[|\Cov(X_i,X_j | X_S)|],$$
and let $\kappa_n$ denote the minimum value i.e. the value of the objective corresponding to $S_n$.  
By repeating the first part of the proof of \cref{thm:new-mean-field}, we get
\begin{align*}
\F_{n} 
&\leq \E_{x_{S_n}}\left[\sum_{ij} (J_n)_{ij}\E_{\mu}[X_i X_j | X_{S_n} = x_{S_n}]] + H_{\nu_{x_{S_n}}}(X)\right] + t_n\\
&\leq \sum_{i,j}(J_n)_{i,j}[\Cov(X_i,X_j|X_{S_n})] + \E_{x_{S_n}}\left[\sum_{i,j}(J_n)_{i,j}\E_{\nu_{x_{S_n}}}[X_iX_j] + H_{\nu_{x_{S_n}}}(X)\right]+t_n.
\end{align*}
As opposed to the proof of \cref{thm:new-mean-field} where we used the Cauchy-Schwarz inequality, here we simply estimate the first term by
\begin{align*}
\sum_{i,j}(J_n)_{i,j}[|\Cov(X_i,X_j|X_{S_n})|]
&\leq 2\binom{n}{2}\frac{\kappa_n \|J_n\|_{\infty}}{\sqrt{t_n}}.
\end{align*}
Finally, set
$$t_n = \min\left\{\frac{n^{4/3}\kappa_{n}^{2/3}\|J_n\|_{\infty}^{2/3}}{\sqrt[3]{4}},n\right\};$$
note that $t_n < n$ for all sufficiently large $n$ by assumption, along with the fact that $\limsup_{n \to \infty}\kappa_{n} \leq \kappa_*$. It follows that for all $n$ sufficiently large,
$$\F_n - \F_n^{*} \leq \frac{3}{\sqrt[3]{4}}n^{4/3}\|J_n\|_{\infty}^{2/3}\kappa_n^{2/3};$$
dividing both sides by $n^{4/3}\|J_n\|_{\infty}^{2/3}$, taking the $\limsup$ as $n \to \infty$, and using $\limsup_{n \to \infty} \kappa_n \leq \kappa_*$ yields the desired conclusion.
\end{proof}
To complete the proof of \cref{thm:refutation-AOZ}, we will exhibit a sequence of Ising models $J_n$ for which $\limsup_{n \to \infty}n\|J_n\|_{\infty}^{2}$ is finite and $\limsup_{n \to \infty}\left(\F_n - \F_n^{*}\right)/\left(n^{4/3}\|J_n\|_{\infty}^{2/3}\right)$ is positive. Specifically, we will show that this is true for a `typical' growing sequence of the Rademacher SK-spin glass. First, we need the following lemma.
%We will show tightness of this bound up to constants using a Rademacher version of the SK spin glass. 
\begin{lemma}
\label{lemma:error-for-spin-glass}
Fix $\beta \in [0,1/2)$. Let $\F_n(\beta)$ denote the (random) free energy of the SK spin glass on $n$ vertices with parameter $\beta$, and let $\F_n^*(\beta)$ denote its variational free energy. Then,
\[ \F_n(\beta) - \F^*_n(\beta) \ge n\beta^2/4 - o(n) \]
asymptotically almost surely (a.a.s) i.e. with probability going to $1$ as $n \to \infty$. This holds under the same universality regime as \cref{thm:rs-correct}.
\end{lemma}
\begin{proof}

\begin{comment}
We first recall a result of Talagrand \cite{talagrand2011mean}
which in fact gives very precise control of lower tails for $\log Z$ in the replica-symmetric regime:
\begin{theorem}[Theorem 11.2.1, \cite{talagrand2011mean}]
For the Sherrington-Kirkpatrick Spin glass at inverse temperature
$\beta < 1$, there exists $K = K(\beta)$ such that for any $t > 0$
\[ \Pr(\log Z_n \le n (\beta^2/4 + \log 2) - t) \le K \exp(-t^2/K). \]
\end{theorem}
\end{comment}
We prove this by calculating $\F_n(\beta)$ and $\F_n^*(\beta)$.
Since $\beta < 1$, we know from \cref{thm:rs-correct} that a.a.s.
\[ \frac{\F_n(\beta)}{n} = \log 2 + \frac{\beta^2}{4} + o_{n}(1). \] 
It remains to compute $\F_n^*(\beta)$. By definition,
\[ \F_n^*(\beta) = \sup_{x \in [-1,1]^n}\left( \frac{\beta}{2} x^T J x + \sum_i H\left(\frac{1 + x_i}{2}\right)\right). \]

We claim that a.a.s., this optimization problem is concave -- indeed, direct calculation shows that for all $x\in [-1,1]$ $$\frac{d^2}{dx^2} H\left(\frac{1 + x}{2}\right) \le -1,$$
whereas Wigner's semicircle law (see \cref{corollary:spectral-norm-bound}) shows that $$\|J\| \le 2 + o_{n \to \infty}(1)$$ a.a.s. Since the Hessian of first term is $J$, this proves the claim since $0\leq \beta < 1/2$. 

Finally, since the gradient of the objective function 
$$\frac{\beta}{2} x^T J x + \sum_i H\left(\frac{1 + x_i}{2}\right)$$
clearly vanishes at the point $x_i = 0$ for all $i\in [n]$, it follows that this point is the global maximizer a.a.s, so that $\F_n^*(\beta) = n\log{2}$ a.a.s. 
\end{proof}
By combining the previous two lemmas, we can prove the following theorem which, in particular, implies \cref{thm:refutation-AOZ}. 
\begin{theorem} Let $\kappa_*$ be the universal constant defined at the start of this section. 
\[ \kappa_* \ge \frac{\sqrt{27}}{16}. \]
\end{theorem}
\begin{proof}
From \cref{lemma:error-for-spin-glass} applied to the Rademacher SK spin glass with parameter $\beta \in [0,1/2)$ i.e. $(J_n)_{ij} = \pm \beta/\sqrt{n}$ independently with probability $1/2$, we obtain a sequence of Ising models $\{J_n\}_{n\geq 1}$ indexed by the number of vertices for which the following holds:
\begin{itemize}
\item $\|J_n\|_{\infty} = \frac{\beta}{\sqrt{n}}$ i.e. $\limsup_{n \to \infty}n\|J_n\|_{\infty}^{2} = \beta^{2}$ 
\item $\limsup_{n\to \infty}\frac{\F_n - \F_n^*}{n\beta^{2}}\geq \frac{1}{4}$ i.e. $\limsup_{n\to \infty}\frac{\F_n - \F_n^*}{n^{4/3}\|J_n\|_{\infty}^{2/3}}\geq \frac{\beta^{4/3}}{4}$. 
\end{itemize}
In view of \cref{lemma:infinity-bound-mean-field}, there are two possibilities: 
\begin{itemize}
\item $\kappa_*^{2} \beta^{2} \geq 16$ for some $\beta \in [0,1/2)$, so that $\kappa_* \geq 4$, or
\item $\kappa_*^{2} \beta^{2} < 16$ for all $\beta \in [0,1/2]$, in which case we have
$$\frac{\beta^{4/3}}{4} \leq \frac{3}{\sqrt[3]{4}}\kappa_*^{2/3}$$
for all $\beta \in [0,1/2)$, so that $\kappa_* \geq \sqrt{27}/16$. 
\end{itemize}
\begin{comment}
each
We take $J$ to be a Rademacher version of the Sherrington-Kirkpatrick model, i.e. $J_{ij} = \beta/\sqrt{n}$ with probability $1/2$ and $-\beta/\sqrt{n}$ otherwise. Then by the previous two lemmas we have for any $\beta \in [0,1/2]$ that
\[ n \beta^2/4 - o(n) \le 2 K_*^{2/3} n \beta^{2/3}. \]
Taking $\beta = 1/2$ and the limit as $n \to \infty$, we get the result.
%\[ \frac{2^{2/3}}{32} \le K_*^{2/3}. \]
\end{comment}
\end{proof}
\section{Mean-field approximation for $k$-MRFs}
\label{sec:mrf}
In this section, we prove a much more general bound for mean-field approximation, extending our result \cref{thm:new-mean-field} to order $k$ Markov random fields (MRFs) over general finite alphabets. Our bound has only a mild dependence on the alphabet size $q$ and is tight for every fixed $k,q$. 
\begin{definition}
An \emph{order $k$ Markov random field} ($k$-MRF) on $n$ vertices over the finite alphabet $\Sigma$ is a probability distribution on the space $\Sigma^n$ of the form
\[ \Pr(X = x) = \frac{1}{Z} e^{f(x) + h(x)}, \]
where the interaction term $f(x)$ can be written as a sum of hyperedge potentials on hyperedges of size $k$ i.e. 
$$f(x) = \sum_{E\subseteq[n],|E|=k} f_E(x_E),$$
and the \emph{external field} $h(x)$ is the sum of the external fields at each vertex i.e. 
\[ h(x) = \sum_{i=1}^{n} h_i(x_i).\]
%and $h_i$ is the \emph{external field at vertex $x_i$}. 
%We say it is a $k$-MRF if we can write the interaction term $f$ as a sum of hyperedge potentials on hyperedges of size $k$, i.e. if 
%\[ f(x) = \sum_{E : |E| = k} f_E(x_E). \]
\end{definition}
In analogy with the Ising model case, we will denote $\sup_{x_E} |f_E(x_E)|$ by $\|f_E\|_{\infty}$ and $\sum_{E\subseteq[n], |E|=k} \|f_E\|_{\infty}^2$ by $\|J\|_F^{2}$.  
The exact same proof as the Ising case gives the following variational principle for the free energy $\F:=\log{Z}$:
\begin{equation}\label{eqn:variational-mrf}
\F = \sup_{\mu}\left[\E_{\mu}[f(x)+h(x)]+H(\mu)\right],
\end{equation}
where the supremum ranges over all probability distributions on $\Sigma^{n}$. By restricting the variational problem to product distributions over $\Sigma^{n}$, we obtain the variational free energy $\F^*$ as before.  
\begin{theorem}
\label{thm:MRF-mean-field}
For any $k$-MRF on $n$ vertices over an alphabet of size $q$,
\[ \F - \F^* \le 3\left(\frac{k \log q}{\sqrt{k!}} n^{k/2} \|J\|_F\right)^{2/3}. \]
\end{theorem}
The proof of this theorem is essentially the same as that of \cref{thm:new-mean-field} with appropriate modifications. We will need the following simple lemma. 
\begin{lemma}\label{lem:tv-lipschitz}
Let $\mu$ and $\nu$ are two probability distributions on the same space $\Omega$. Then for any function $f \colon \Omega \to \R$ such that $|f(X)| \le M$ a.s. under both $\mu$ and $\nu$, we have
\[ \left|\E_{X \sim \mu}[f(X)] - \E_{Y \sim \nu}[f(Y)]\right| \le 2M \TV(\mu, \nu). \]
\end{lemma}
\begin{proof}
By a standard characterization of \TV, we can couple $X$ and $Y$ so that $\Pr(X \ne Y) = \TV(\mu,\nu)$. Since $|f(X) - f(Y)| \le 2M$ a.s, we are done. 
\end{proof}
\begin{proof}[Proof of \cref{thm:MRF-mean-field}]
Let $\epsilon > 0$ be some parameter which will be optimized later. We begin by applying \cref{thm:correlation-rounding} with $\ell = 1/(\epsilon^2 \log{q})$; let $S$ be the resulting set of size at most $\ell$. Let $\mu$ denote the Boltzmann distribution. For each assignment $x_S \in \Sigma^{|S|}$ to the variables in $S$, let $\nu_{x_S}$ denote the product measure on $\Sigma^{n}$ for which $\E_{\nu_{x_S}}[X_i] = \E[X_i|X_S = x_S]$. Then, using the variational principle, the same computation as in the binary Ising model case shows that
$$\F \leq \E_{x_S}\left[\E_{\mu}[f(X)|X_S=x_S]+\E_{\nu_{x_S}}[h(X)]+H_{\nu_{x_S}}(X)\right] + \ell \log{q}.$$
As before, we decompose the first term as
$$\E_{x_S}\left[\E_{\mu}[f(X)|X_S=x_S\right] = \E_{x_S}\left[\E_{\nu_{x_S}}[f(X)] + \left[\E_{\mu}[f(X)|X_S = x_S]-\E_{\nu_{x_S}}[f(X)]\right]\right].$$
Since $f(X) = \sum_{E \in \binom{[n]}{k}}f(X_E)$, it follows by \cref{lem:tv-lipschitz} that
\begin{align*}
\E_{x_S}\left|\E_{\mu}[f(X)|X_S = x_S]-\E_{\nu_{x_S}}[f(X)]\right| \leq 2\binom{n}{k}\E_{x_S}\E_{E \sim \binom{[n]}{k}}\left[\|f_E\|_\infty \TV\left((\mu|X_S=x_S)|_{X_E},\nu_{x_S}|_{X_E}\right)\right].
\end{align*}
By the Cauchy-Schwarz inequality, the right hand side is bounded by
$$2 \sqrt{\binom{n}{k}} \|J\|_{F}\sqrt{\E_{x_S}\E_{E \sim \binom{[n]}{k}}\TV^{2}\left((\mu|X_S=x_S)|_{X_E},\nu_{x_S}|_{X_E}\right)},$$
whereas by Pinsker's inequality and the choice of $S$, we have
\begin{align*}
\sqrt{\E_{x_S}\E_{E \sim \binom{[n]}{k}}\TV^{2}\left((\mu|X_S=x_S)|_{X_E},\nu_{x_S}|_{X_E}\right)}
&\leq \sqrt{\E_{E \sim \binom{[n]}{k}}C\left((\mu|X_S)|_{X_E},\nu_{X_S}|_{X_E}\right)}\\
&\leq \frac{k\sqrt{\log{q}}}{\sqrt{\ell}}.
\end{align*}
To summarize, there exists some $x_S$ such that the associated product distribution $\nu:=\nu_{x_S}$ satisfies
\begin{equation}
\label{eqn:bound-with-epsilon}
\F \leq \E_{\nu}[f(x)+h(x)]+H(\nu) + 2k\epsilon \sqrt{\binom{n}{k}} \|J\|_{F}\log{q} + \frac{1}{\epsilon^2}.
\end{equation}
Using ${n \choose k} \le n^k/k!$
and optimizing the value of $\epsilon$ completes the proof.
\end{proof}
 
\subsection{Tightness of \cref{thm:MRF-mean-field}}
In our formulation, there is a natural way to lift a $k$-MRF to an $\ell$-MRF for any $k \le \ell$ by the following averaging procedure. Given a $k$-MRF specified by the collection $(f_E)_{E \in \binom{[n]}{k}}$, we define the collection of functions $(g_F)_{F \in \binom{[n]}{\ell}}$ by
\[ g_F(x_F) := \frac{1}{{n - k \choose \ell - k}} \sum_{E \subset F, |E| = k} f_E(x_E). \]
This scaling is chosen so that for any $x$,
\[ \sum_{F \in \binom{[n]}{\ell}} g_F(x_F) = \sum_{F \in \binom{[n]}{\ell}}\frac{1}{{n - k \choose \ell - k}} \sum_{E \subset F, |E| = k} f_E(x_E) = \sum_{E \in \binom{[n]}{k}} f_E(x_E).\]
Hence, both the $k$-MRF and the $\ell$-MRF correspond to the same distribution over $\Sigma^{n}$, and thus have the same mean-field error. On the other hand, it follows from the triangle inequality that 
\[ \sum_{F \in \binom{[n]}{\ell}} \|g_F(x_F)\|_{\infty}^2 \le \left(\frac{{\ell \choose k}}{{n - k \choose \ell - k}}\right)^2 \sum_{F \in \binom{[n]}{\ell}} \sum_{|E| = k, E \subset F} \|f_E(x_E)\|_{\infty}^2 = \frac{{\ell \choose k}^2}{{n - k \choose \ell - k}} \sum_{E \in \binom{[n]}{k}} \|f_E(x_E)\|_{\infty}^2 \]
In particular, denoting $\sum_{F \in \binom{[n]}{\ell}} \|g_F(x_F)\|_{\infty}^2$ by $\|J_\ell\|_{F}^{2}$ and $\sum_{E \in \binom{[n]}{k}} \|f_E(x_E)\|_{\infty}^2$ by $\|J_k\|_{F}^{2}$, we see that for $k$ and $\ell$ fixed,
$$\|J_\ell\|_{F}^{2/3} \leq \frac{C_{k,\ell}}{n^{\ell - k}}\|J_k\|_{F}^{2},$$
so that
$$n^{\ell/3}\|J_\ell\|_{F}^{2/3} \leq C_{k,\ell}n^{k/3}\|J_k\|_{F}^{2/3}.$$
%$J_F^2$ scales like $C_{k,\ell}/n^{\ell - k}$ under lifting from $k$ to $\ell$. 
Therefore by lifting any of the tight examples for \cref{thm:new-mean-field}, we get a corresponding tightness result for $k$-MRFs:
\begin{theorem}
For fixed $k$ and $q$, \cref{thm:MRF-mean-field} is tight up to constants. In other words, there exists an absolute constant $c_{k,q} > 0$ such that for infinitely many $k$-MRFs on an alphabet of size $q$, 
\[ \F - \F^* \ge c_{k,q} \left(n^{k/2} \|J\|_F\right)^{2/3}.\]
\end{theorem}
\begin{remark}
This tightness guarantee for mean-field also shows that \cref{thm:correlation-rounding} is tight up to constants for any fixed $k$. No more general form of \cref{conjecture:AOZ} was given for higher-order models, but combining the lifting result with the construction from \cref{thm:refutation-AOZ} gives an analogous tightness result in terms of average TV-distance between product and joint distributions, ruling out improved bounds.
\end{remark}
\section{Algorithmic results: proof of Theorem~\ref{thm:algorithm-imprecise}}
\label{s:algorithmic}
We now show how to go from the proof of our bounds on the quality of mean-field approximation to concrete algorithms; this is a relatively straightforward application of the Sherali-Adams relaxation. The only serious difficulty is to find a good proxy for the entropy that is suitable for use with pseudo-distributions; this was solved in \cite{risteski2016calculate} by introducing the following \emph{pseudo-entropy functional} for level $(r + 1)$ pseudo-distributions:
\begin{equation}
\tilde{H}_r(\mu) = \min_{S : |S| \le r} \left[H(X_S) + \sum_{i} H(X_i | X_S)\right].
\end{equation}
By the chain rule for entropy, we see that for any $r$ and for any true probability distribution $\mu$, $H(\mu) \le \tilde{H}_r(\mu)$. Moreover, essentially the standard proof of the concavity of entropy shows that for any $r$, $\tilde{H}_r(\mu)$ is a concave function of the pseudo-distribution $\mu$ (Lemma 8 of \cite{risteski2016calculate}).
Then, we can write the Sherali-Adams relaxation to \cref{eqn:variational-mrf} as
\begin{equation}\label{eqn:risteski-relaxation}
\F_{SA, r + k} := \max_{\mu \in SA_{r + k}} \tilde{E}[f(X) + h(X)] + \tilde{H}_r(\mu). 
\end{equation}
Note that by considering the Boltzmann distribution $\mu$ in the above optimization problem, and using that $H(\mu) \leq \tilde{H}_{r}(\mu)$, it follows that $\F_{SA,r+k} \geq \F$.

Combining this relaxation with correlation rounding gives Algorithm \textsc{SA-MeanField} for finding good mean-field solutions.
\begin{algorithm}
\caption{\textsc{SA-MeanField}}
\begin{enumerate}
\item Find a pseudo-distribution $\mu$ maximizing \cref{eqn:risteski-relaxation} within $\epsilon$ additive error. This can be done efficiently using (for example) the ellipsoid method.
\item For every $S \subseteq [n]$ with $|S| \le r$ and for every $x_S \in \Sigma^{S}$, let $\nu_{S,x_S}$ be the product distribution given by matching the first moments of $\mu$ conditioned on $X_S = x_S$.
\item Return the $\nu_{S,x_S}$ which maximizes $\E_{\nu}[f(X) + h(X)] + H(\nu)$.
\end{enumerate}
\end{algorithm}
\begin{remark}
Instead of searching over all $S\subseteq [n]$ with $|S|\leq r$, we may greedily select $S$ vertex by vertex, stopping when the average total correlation $\E_{E}[C(X_E | X_S)]$ satisfies the guarantee of \cref{thm:correlation-rounding}. That this works follows from a slightly modified analysis of correlation rounding.
\end{remark}
\begin{theorem}\label{thm:sa-meanfield}
Let $H(p)$ denote the entropy of $\Ber(p)$. We have the following running time and performance guarantees for Algorithm \textsc{SA-MeanField}. 
\begin{enumerate}
\item The running time is
\[ 2^{O(n H((r+k)/n) + (r+k) \log q)} + \poly\log(1/\epsilon).\]
\item The product distribution $\nu$ returned by the algorithm satisfies
\[ 0 \le \F - \F_{\nu} \le \sqrt{\frac{4\log{q}}{r}} \frac{k n^{k/2} \|J\|_F}{\sqrt{k!}} + r \log q + \epsilon, \]
where
\[ \F_{\nu} := \E_{\nu}[f(X) + h(X)] + H(\nu). \]
\item We also have the following guarantee for the pseudo-distribution $\mu$ computed in the first step:
\[ 0 \le \F_{SA, r + k}(\mu) - \F \le \sqrt{\frac{4\log{q}}{r}} \frac{k n^{k/2} \|J\|_F}{\sqrt{k!}} + \epsilon,\]
where
$$\F_{SA, r + k}(\mu) := \tilde{E}_\mu[f(X)+h(X)] + \tilde{H}_{r}(\mu).$$
\end{enumerate}
\begin{comment}
Algorithm \textsc{SA-MeanField} runs in time
\[ 2^{O(n H(r/n) + r \log q)} + poly(n,\log(1/\epsilon))\]
where $H(p)$ is the entropy of $Ber(p)$,
and returns a product distribution $\nu$ such that
\[ 0 \le \F - \F_{\nu} \le \frac{1}{\sqrt{r}} \frac{2 k n^{k/2} \|J\|_F}{\sqrt{k!}} + r \log q + \epsilon. \]
where
\[ \F_{\nu} := \E_{\nu}[f(X) + h(X)] + H(\nu) \]
We also have the following tightness guarantee for the SA relaxation:
\[ 0 \le \F_{SA, r + 2} - \F \le \frac{1}{\sqrt{r}} \frac{2 k n^{k/2} \|J\|_F}{\sqrt{k!}} \]
\end{comment}
\end{theorem}

\begin{proof}
The runtime is dominated by the first step, where we solve a convex program with at most $q^{r+k} {n \choose r+k}$ many variables and $\poly\left(q^{r+k}\binom{n}{r+k}\right)$ many LP constraints. Therefore, by standard guarantees for the ellipsoid method \cite{gls} we can solve \cref{eqn:risteski-relaxation} within $\epsilon$ additive error in time $\poly\left(q^{r+k}\binom{n}{r+k}, \log(1/\epsilon)\right)$. Using the standard bound (which follows from sub-additivity of entropy)
\[ \log {n \choose r+k} \le n H\left(\frac{r+k}{n}\right), \]
this quantity is at most $\poly\left(2^{O(nH((r+k)/n)+(r+k)\log{q})},\log(1/\epsilon)\right)$. 
Finally, we use the AM-GM inequality to separate the $2^{O(n H((r+k)/n) + (r+k) \log q)}$ term in the bound.

For $2.$, note that $0\leq \F - \F_{\nu}$ follows from the Gibbs variational principle, so we only need to show the right inequality. We will deduce this from the stronger (since $\F_{SA,r+2} \geq \F$) statement
\begin{equation}
\F_{SA,r+2} - \F_{\nu} \leq \sqrt{\frac{4\log{q}}{r}} \frac{k n^{k/2} \|J\|_F}{\sqrt{k!}} + r \log q + \epsilon,
\end{equation}
which itself follows from
\begin{equation}
\label{eqn:alg-cor-round-pseudo}
\F_{SA,r+2}(\mu) - \F_{\nu} \leq \sqrt{\frac{4\log{q}}{r}} \frac{k n^{k/2} \|J\|_F}{\sqrt{k!}} + r \log q,
\end{equation}
where $\mu$ is the $r+2$ pseudo-distribution returned in the first step. Now, note that \cref{eqn:alg-cor-round-pseudo} follows by exactly the same proof as for \cref{thm:MRF-mean-field} (in particular, \cref{eqn:bound-with-epsilon}) using the fact that an $r+k$ pseudo-distribution suffices to give the correlation rounding guarantee on sets of size at most $r$, and recalling that in \cref{eqn:bound-with-epsilon}, $\epsilon = 1/\sqrt{r\log q}$. 
%The output guarantee for \textsc{SA-MeanField} follows from the fact that $\F \le \F_{SA, r + 2}$ and by following the proof of \cref{thm:MRF-mean-field} until the last step: because correlation rounding is local, an $r + 2$ pseudo-distribution suffices to give the correlation rounding guarantee on sets up to size $r$.  

Finally, $3.$ follows from \cref{eqn:alg-cor-round-pseudo}, noting additionally that we can avoiding losing the term $r \log q$ (equivalently, the term $1/\epsilon^{2}$ in \cref{eqn:bound-with-epsilon}), if we round instead to the mixture of product distributions given by $\sum_{x_S} P(x_S) \nu_{S,x_S}$.
\end{proof}
In particular, we obtain the following more general and precise version of \cref{thm:algorithm-imprecise}.
\begin{corollary}
\label{corollary:algorithm-precise}
Fix $k$ and $q$. If $\|J_n\|_F \leq c_{k,q}f(n)n^{3/2 - k/2}$, where $f(n) \to 0$ as $n \to \infty$ and $c_{k,q}>0$ is some constant depending only on $k$ and $q$, then $\F_n$ can be approximated to within $\sqrt{f(n)} n$ additive error in (sub-exponential) time $2^{-O\left(n\sqrt{f(n)}\log{f(n)}\right)}$ by Algorithm \textsc{SA-MeanField}. Moreover, the algorithm outputs a product distribution achieving this approximation.
\end{corollary}
\subsection{Faster algorithms using random subsampling}
Until now, the algorithms we considered have been deterministic. However, in dense instances there is a major advantage to using randomness: we can accurately estimate $\F$ by looking at a \emph{vanishingly small portion} of the entire input instance. In \cite{second-paper} the following structural guarantee is given, relating the free energy of small random induced subgraphs to that of the original model: Fix a $k$-MRF on the vertex set $[n]$ with interaction functions $(f_E)_{E\in \binom{[n]}{k}}$, and denote its free energy by $\F$. Consider a random subset $Q$ of $[n]$ of size $|Q|=s$. Consider also the $k$-MRF on the vertex set $Q$ whose interaction functions are given by $$\left(\frac{n^{k-1} f_E}{s^{k-1}}\right)_{E \in \binom{Q}{k}}.$$  
We will denote the free energy of this $k$-MRF by $\F_Q$.
\begin{theorem}[Theorem 4, \cite{second-paper}]
Let $\epsilon > 0$ and suppose $s \ge 10^6\omega$, where $\omega:= k^7\log(1/\epsilon)/\epsilon^{8}$. 
%Let $\omega := \log(1/\epsilon)/\epsilon^8$ (so that $\omega/q < 1/128000$). 
Then, with probability at least $39/40$:
$$\left|\F - \frac{n}{s}\F_Q\right| \leq C_q k^4 \epsilon  \left(n^{k/2}\|J\|_F + \epsilon n^{k} \|J\|_{\infty} + \omega n/s\right),$$
where $\|J\|_{\infty} := \sup_E \|f_E\|_{\infty}$.
\end{theorem}
Note that for the (rescaled) sampled $k$-MRF, it follows from Markov's inequality that 
$$\|J_Q\|^{2}_F \leq 10\frac{n^{2k-2}\binom{s}{k}}{s^{2k-2}\binom{n}{k}}\|J\|^{2}_F \leq 10e^k\left(\frac{n}{s}\right)^{k-2} \|J\|^{2}_F$$ with probability at least $9/10$. Whenever this happens, \cref{thm:sa-meanfield} shows that we can estimate $n\F_Q/s$ to within additive error 
$$\sqrt{\frac{40\log{q}}{r}} \frac{ke^{k/2} n^{k/2} \|J\|_F}{\sqrt{k!}} + \frac{n\varepsilon}{s} \leq 10\sqrt{\frac{\log{q}}{r}} {n^{k/2} \|J\|_F} + \frac{n\varepsilon}{s} $$ 
in time $2^{O(s H((r+k)/s) + (r+k) \log q)} + \poly\log(1/\varepsilon)$. Taking $r = 1/(\epsilon^{2}\log q)$ and $\varepsilon = \epsilon$, it follows that with probability at least $7/8$, we can find an estimate $\hat{\F}$ to $\F$ in \emph{constant} time $2^{O_{k,q}\left(\frac{1}{\epsilon^2}\log(\frac{1}{\epsilon})\right)} $ such that 
$$\left|\F - \hat{\F}\right|\leq C_qk^4 \epsilon \left(  n^{k/2} \|J\|_F + \epsilon n^{k} \|J\|_{\infty} + \omega n/s\right).$$
Given an error probability $\delta > 0$, by repeating the above procedure independently $O(\log(1/\delta))$ many times and returning the median estimate, the standard Chernoff bound allows us to obtain the following. 
\begin{theorem}
Let $\delta,\epsilon > 0$ and suppose $s \ge 10^6\omega$, where $\omega:= k^7\log(1/\epsilon)/\epsilon^{8}$. 
%Let $\omega := \log(1/\epsilon)/\epsilon^8$ (so that $\omega/q < 1/128000$). 
Then,  the above algorithm runs in time $2^{O_{k,q}\left(\frac{1}{\epsilon^2}\log(\frac{1}{\epsilon})\right)}\log(1/\delta)$  and returns an estimate $\hat{\F}$ such that:
$$\left|\F - \hat \F \right| \leq C_q k^4 \epsilon\left(n^{k/2}\|J\|_F + \epsilon n^{k} \|J\|_{\infty} + \omega n/s\right)$$
with probability at least $1 - \delta$.
\end{theorem}

\subsection{Algorithmic tightness under Gap-ETH}
It's natural to ask if the tradeoff between graph density (more precisely, $\|J\|_F$) and runtime in our algorithm is optimal. It turns out that under a variant of the \emph{Exponential Time Hypothesis}, this is indeed true. The variant we need is the following conjecture known as ETHA or Gap-ETH \cite{manurangsi2017birthday}:
\begin{conjecture}[Gap-ETH]
There exist constants $\epsilon, c > 0$ such that no algorithm running in time $O(2^{cn})$ can distinguish between a satisfiable $3$-SAT formula and a $3$-SAT formula with at most $1 - \epsilon$ fraction of satisfiable clauses. Here, $n$ denotes the number of clauses.
\end{conjecture}
One of the motivations for this conjecture is that under the ordinary ETH, the quasilinear-length PCP of Dinur \cite{dinur2007pcp} shows that there exists some $\epsilon > 0$ such that no algorithm running in time $\Omega(2^{n/\poly\log(n)})$ can distinguish between a satisfiable $3$-SAT formula and one with at most $1-\epsilon$ fraction of satisfiable clauses; if this PCP were of linear-length, then one could deduce Gap-ETH from ETH. Under Gap-ETH, one immediately finds that $\|J\|_F^2 = o(n)$ is the tight regime for approximating $\F/n$ with sub-exponential time algorithms.
\begin{proposition}\label{prop:max-cut-eth}
Under Gap-ETH, the following holds for some $\epsilon> 0$:
\begin{enumerate}
\item There exist a constant $c>0$ and an infinite family of graphs with $\Theta(n)$ many edges on which it takes time at least $2^{cn}$ to approximate MAX-CUT within multiplicative error $(1 - \epsilon)$.
\item There exist a constant $c>0$ and an infinite family of Ising models with $\|J\|_F^2 = \Theta_{\epsilon}(n)$ on which it takes time at least $2^{cn}$ to approximate $\F$ within additive error $\epsilon n$.
\end{enumerate}
\end{proposition}
\begin{proof}
1. This follows directly from the statement of Gap-ETH and the existence of an $L$-reduction from MAX-3SAT to MAX-CUT \cite{papadimitriou1991optimization}.

2. This follows from (1) by defining the corresponding anti-ferromagnetic Ising model and sufficiently high inverse temperature $\beta$, which gives an approximation guarantee for MAX-CUT as in \cref{eqn:max-cut-ising}. 
\end{proof}
\begin{remark}
Complexity-theoretic bounds straightforwardly imply lower bounds on the number of Sherali-Adams rounds needed; for example \cref{prop:max-cut-eth} implies that for these graphs $\Omega(n)$ rounds of Sherali-Adams are needed to approximate MAX-CUT; if, on the contrary, only $o(n)$ rounds sufficed, then solving the LP would give a $2^{n H(o(n)/n)} = 2^{o(n)}$ time algorithm (see \cref{thm:sa-meanfield}).
\end{remark}
We can further apply reductions from \cite{fotakis2015sub} to get additional tightness results; they originally stated their results under the assumption of ETH, but the same reductions can be applied from Gap-ETH as well and give the following cleaner results. 
\begin{theorem}[\cite{fotakis2015sub}]\label{thm:fotakis}
Under Gap-ETH, there is some $\epsilon > 0$ for which the following holds.
\begin{enumerate}
\item Consider an arbitrary sequence $d_n$ with $d_n = o(n)$. Then there does not exist any algorithm which approximates MAX-CUT within multiplicative error $(1 - \epsilon)$ in time $2^{o(n/d_n)}$ on all graphs of average degree at least $d_n$.
\item There exist a constant $c>0$ and an infinite family of $k$-SAT instances with $\Theta_k(n^{k - 1})$ many clauses (all of which are distinct) on which it takes time at least $2^{cn}$ to approximate MAX-$k$-SAT within multiplicative error $(1 - \epsilon)$. 
\end{enumerate}
\end{theorem}
As with \cref{prop:max-cut-eth}, these translate immediately to
lower bounds for computing partition functions by picking a sufficiently
large inverse temperature $\beta$:
\begin{corollary}
\label{corollary:time/density-tradeoff}
Under Gap-ETH, there is some $\epsilon > 0$ such that
\begin{enumerate}
\item Fix any sequence $d_n = o(n)$. There is no algorithm which computes $\F$ within additive $\epsilon n$ error in time $2^{o(d_n)}$ on Ising models
where $\|J\|_F^2 \le d_n$.
\item For any fixed $k \ge 2$, there exist a constant $c>0$ and  an infinite family of binary $k$-MRFs with $\|J\|_F = \Theta_{k}(n^{3/2 - k/2})$ on which it takes time at least $2^{cn}$ to approximate $\F$ within $\epsilon n$ additive error. 
\end{enumerate}
\end{corollary}
\begin{proof}
(1) follows directly from \cref{thm:fotakis} using the same reduction as in \cref{prop:max-cut-eth}. A slight generalization of this argument also shows (2): consider $\epsilon > 0$ and a family of $k$-SAT instances on $n$ variables and $m_n = \Theta_k(n^{k-1})$ (distinct) clauses as in part (2) of \cref{thm:fotakis}. For the reduction, we start from the $k$-SAT instance with $n$ variables and $m$ distinct clauses, and define for each $E \in \binom{[n]}{k}$ 
\[ f_E(x_E) := \frac{\beta n}{m} \#\{\text{clauses depending only on the variables in $E$ which are satisfied by $x_E$}\}, \]
where $\beta$ is a sufficiently large constant (depending on $\epsilon$) to be specified later. 
Hence,
\[ \|J\|_F^2 := \sum_E \|f_E\|_{\infty}^2 \le \frac{\beta^2 n^2}{m} 2^{2k} \]
since there are at most $2^k$ distinct clauses supported on $x_E$ and at most $m$ subsets $E$ which support a clause. 
Therefore, if we assume that (2) is false, then for any $c > 0$, we can compute the free energy of this model within additive error $n$ in time at most $2^{cn}$ as long as 
\[ \frac{\beta^2 n^2}{m} 2^{2k} = \Theta_k(n^{3 - k}), \]
which is true since $m = \Theta_k(n^{k - 1})$ by assumption. On the other hand, since 
$$\sum_E f_E(x_E) = \frac{\beta n}{m} \#\{\text{satisfied clauses for assignment $x$}\},$$
and since there is at least one assignment $x$ for which the number of clauses satisfied is at least $m(1-2^{-k}) $, it follows that if we take $\beta = 1/4\epsilon$, then an $n$-additive approximation for the partition function gives an $\epsilon n$-additive approximation for the $k$-SAT instances (by returning the approximation to the partition function multiplied by $m/n\beta$), thereby contradicting part (2) of \cref{thm:fotakis}. 

\end{proof}
\section{Conclusion}

We presented a unified perspective on two major variational approaches to calculating the free energy that hitherto seemed completely disparate: mean-field approximations and convex relaxations. This view has both analytic benefits (we derived bounds on the quality of mean-field approximations) and algorithmic benefits (we derived algorithms for approximating the free energy up to the intractability limit).  

We conclude with several open problems:
\begin{enumerate}
\item As mentioned earlier, there is a straightforward example showing that up to a constant, the exponent $\frac{2}{3}$ is optimal in \cref{thm:new-mean-field} for the natural univariate quantity $(n\|J\|_F)$. However, this example does not rule out other bounds of the form $O(n^{1-\alpha}\|J\|^{2\alpha}_F)$ for $\alpha \in [0,1]$. As there is always a trivial bound $O(n)$ for the mean-field approximation (consider the optimal point-mass distribution), we may assume that $\|J\|_F = o(n^{1/2})$ and ask about the supremum of all $\alpha$ such that an upper bound of this form holds. The Curie-Weiss model at critical temperature shows that we cannot take $\alpha$ to be $0$ without introducing additional logarithmic factors in the upper bound. Other than this, we have unfortunately not been able to rule out any other values of $\alpha \in (0,1/3]$. 

\item It's possible that the fRSB phase of the SK spin glass is more difficult to correlation-round than the RS phase. Indeed, the landscape picture for the fRSB phase seems like a natural obstruction to correlation rounding and was what originally motivated us to consider spin glasses. Is one of these spin glass models \emph{extremal}, in the sense that they can be used to get the optimal value of $\kappa_*$? If not, what do the extremal distributions look like? 
\item How many rounds do convex hierarchies (Sherali-Adams, Sum-of-Squares) need to correctly estimate the value of the free energy and ground state of the SK spin glass? (By computing the ground state, we mean to drop the entropy and just consider the MAX-QP problem.) Are $\Omega(n)$ rounds required?

\end{enumerate}

\section{Acknowledgements}
We would like to thank Elchanan Mossel for stimulating discussions and Ankur Moitra for providing valuable feedback on the paper. 
\bibliographystyle{apalike}
\bibliography{all,ising-regularity}
\appendix

\section{Appendix: Proof of \cref{thm:correlation-rounding}}
We will make use of the following information theoretic notion:
\begin{definition}
The \emph{multivariate mutual information} of a collection of random variables $X_1,\dots,X_n$ is defined to be 
\[ I(X_1; \cdots; X_n) = \sum_{m = 1}^n (-1)^{m-1} \sum_{S \subset {n \choose m}} H(X_S). \]
\end{definition}
Note that when $n=2$, this corresponds to the usual notion of mutual information between two random variables. We may also define the \emph{conditional multivariate mutual information} by using the conditional entropy in the above equation; note that the chain rule for entropy shows immediately that 
\[ I(X_1; \cdots; X_n) = I(X_1; \cdots; X_{n - 1}) - I(X_1; \cdots; X_{n - 1} | X_n). \]
\\
\label{appendix:proof-of-corr-rounding}
We will deduce \cref{thm:correlation-rounding} from the following lemma, which is slightly stronger. Our statement and proof correct two errors found in \cite{manurangsi2017birthday,yoshida-zhou}: missing sign terms in the relation between $C(X_S)$ and $I(X_S)$, and use of an invalid version of identity \cref{eqn:corr-info} below which sums over tuples instead of sets.
\begin{lemma}\label{lem:correlation-rounding-V-S}
Let $X_1,\dots,X_n$ be a collection of $\{\pm 1\}$-valued random variables. Then, for any $k,\ell \in [n]$, there exists some $t\leq \ell$ such that:  
\[ \E_{S \sim \binom{V}{t}} \E_{F \sim \binom{V-S}{k}}[C(X_F | X_S)] \le \frac{k^2 \log(2)}{\ell}.\]
\end{lemma}
\begin{proof}
We begin by showing that
\begin{equation}
\label{eqn:corr-info}
\E_{F\sim{V \choose k}}[C(X_{F}|X_S)] = \sum_{r=2}^{k}{k \choose r}(-1)^{r}\E_{R\sim{V \choose r}}[I(X_{R}|X_S)].
\end{equation}
For simplicity, we will prove the unconditional version of this identity. The same proof gives the conditional version as well. We start by noting that:
\[ C(X_1; \cdots; X_n) = \sum_{R \subset [n], |R| \ge 2}(-1)^{|R|} I(X_R).\]
Therefore, 
\begin{align*}
\sum_{F\subseteq{V \choose k}}C(X_{F}) & =\sum_{S\subseteq{V \choose k}}\sum_{R\subseteq F,|R|\geq2}(-1)^{|R|}I(X_{R})\\
 & =\sum_{r=2}^{k}\sum_{R\subseteq{V \choose r}}{|V|-r \choose |V|-k}(-1)^{r}I(X_{R}),
\end{align*}
and dividing both sides by $\binom{|V|}{k}$ gives: 
\begin{align*}
\E_{F\sim{V \choose k}}[C(X_{F})] & =\sum_{r=2}^{k}{|V|-r \choose k-r}{|V| \choose r}{|V| \choose k}^{-1}(-1)^{r}\E_{R\sim{V \choose r}}[I(X_{R})]\\
 & =\sum_{r=2}^{k}{k \choose r}(-1)^{r}\E_{R\sim{V \choose r}}[I(X_{R})],
\end{align*}
as desired. 
%using the fact that every set of size $k$ contains exactly ${k \choose r}$ subsets of size $r$, it follows by double counting that 
%\[ \E_{S \sim V^k}[C(X_S)] = \sum_{r = 2}^k {k \choose r} \E_{R \sim V^r}[I(X_R)]\]

Next, we consider the key quantity: 
\[ Q := \sum_{t = 0}^{\ell} \E_{S \sim \binom{V}{t}} \E_{F \sim \binom{V-S}{k}}[C(X_F | X_S)] = \sum_{r = 2}^k {k \choose r}(-1)^{r} \sum_{t = 0}^{\ell} \E_{S \sim \binom{V}{t}} \E_{R \sim \binom{V-S}{r}}[I(X_R | X_S)], \]
where the second equality follows from \cref{eqn:corr-info}. By the chain rule for mutual information, we have the telescoping sum:
\begin{align*}
\sum_{t = 0}^{\ell} \E_{S \sim \binom{V}{t}} \E_{R \sim \binom{V-S}{r}}[I(X_R | X_S)] 
&= \sum_{t = 0}^{\ell} \left(\E_{S \sim \binom{V}{t}}\E_{E \sim \binom{V-S}{r - 1}}[I(X_E | X_S)]  - \E_{S \sim \binom{V}{t + 1}} \E_{E \sim \binom{V-S}{r - 1}}[I(X_E | X_S)]\right) \\
&= \E_{E \sim \binom{V}{r - 1}}[I(X_E)] - \E_{S \sim \binom{V}{\ell + 1}} \E_{E \sim \binom{V-S}{r - 1}}[I(X_E | X_S)],
\end{align*}
so that
\begin{align*}
Q 
&= \sum_{r = 2}^k {k \choose r}(-1)^{r}\left(\E_{E \sim \binom{V}{r - 1}}[I(X_E)] - \E_{S \sim \binom{V}{\ell + 1}} \E_{E \sim \binom{V-S}{r - 1}}[I(X_E | X_S)]\right) \\
&\le  {k \choose 2} \E_{i \sim V}[H(X_i)] + \sum_{r = 3}^k {k \choose r}(-1)^{r}\left(\E_{E \sim \binom{V}{r - 1}}[I(X_E)] - \E_{S \sim \binom{V}{\ell + 1}} \E_{E \sim \binom{V-S}{r - 1}}[I(X_E | X_S)]\right),
\end{align*}
where in the second line, we have separated out the $r = 2$ term, and dropped the nonpositive term $-\binom{k}{2}\E_{S \sim \binom{V}{\ell + 1}}\E_{i \sim V-S}[H(X_i|X_S)]$. 

Now, recall that
\[ {k \choose r} = {k - 1 \choose r - 1} + {k - 2 \choose r - 1} + \cdots + {r - 1 \choose r - 1}.\]
Hence,
\begin{align*}
Q &\le {k \choose 2}\E_{i\sim V}[H(X_i)] - \sum_{d = 2}^{k - 1}\sum_{r = 3}^{d + 1} (-1)^{r-1}{d \choose r - 1}\left(\E_{E \sim \binom{V}{r - 1}}[I(X_E)] - \E_{S \sim \binom{V}{\ell + 1}} \E_{E \sim \binom{V-S}{r - 1}}[I(X_E | X_S)]\right) \\
&= {k \choose 2}\E_{i \sim V}[H(X_i)] - \sum_{d = 2}^{k - 1} \left(\E_{F \sim \binom{V}{d}}[C(X_F)] - \E_{S \sim \binom{V}{\ell + 1}}\E_{F \sim \binom{V-S}{d}}[C(X_F | X_S)]\right) \\
&\le {k \choose 2}\E_{i \sim V}[H(X_i)] + \sum_{d = 2}^{k - 1} \E_{S \sim \binom{V}{\ell+1}}\E_{F \sim \binom{V-S}{d}}[C(X_F)] \\
&\le {k \choose 2}\E_{i \sim V}[H(X_i)] + \sum_{d = 2}^{k - 1} \E_{S\sim \binom{V}{\ell+1}}\E_{F \sim \binom{V-S}{d}}\left[\sum_{i \in F} H(X_i)\right] \\
&\le \left({k \choose 2} + \sum_{d=2}^{k-1}d\right)\E_{i \sim V}[H(X_i)]\\
&\le k^2 \log(2),
\end{align*}
where we have used \cref{eqn:corr-info} in the second line.
Recalling the definition of $Q$, we see that there exists some $t\in \{0,1,\dots,\ell\}$ such that 
\[\E_{S \sim \binom{V}{t}} \E_{F \sim \binom{V-S}{k}}[C(X_F | X_S)] \le \frac{k^2 \log(2)}{\ell}.\]
\end{proof}
In order to deduce \cref{thm:correlation-rounding} from this lemma, we need the following two simple properties of the total correlation. 
\begin{itemize}
\item For any $F,S \subseteq[n]$, $C(X_F|X_S) = C(X_{F\cap S^c}|X_S)$. This follows since by the chain rule for entropy 
\begin{align*}
C(X_F|X_S) 
&= \sum_{j\in F}H(X_j|X_S) - H(X_F|X_S)\\
&= \sum_{j\in F\cap S^c}H(X_j|X_S) - H(X_{F\cap S}|X_S) - H(X_{F\cap S^c}|X_S)\\
&= \sum_{j\in F\cap S^c}H(X_j|X_S) - H(X_{F\cap S^c}|X_S)\\
&= C(X_{F\cap S^c}|X_S).
\end{align*}
\item For any $S\subseteq [n]$ and $F\subseteq E \subseteq [n]$, $C(X_F|X_S) \leq C(X_E|X_S)$. Indeed, by the chain rule for entropy, and since conditioning decreases entropy
\begin{align*}
C(X_E|X_S) 
&= \sum_{i \in E}H(X_i|X_S) - H(X_E|X_S)\\
&= \left[\sum_{i \in F}H(X_i|X_S) - H(X_F|X_S)\right] + \left[\sum_{i \in E\setminus F}H(X_i|X_S) - H(X_{E\setminus F}|X_{S\cup F})\right]\\
&\geq C(X_F|X_S) + C(X_{E\setminus F|X_{S\cup F}}).
\end{align*}
\end{itemize}
\begin{proof}[Proof of \cref{thm:correlation-rounding}]
Fix an arbitrary $S \in {V \choose t}$. We will show that
\begin{equation}\label{eqn:averaging-C}
\E_{F \sim {V \choose k}}[C(X_F | X_S)] \le \E_{E \sim {V - S \choose k}}[C(X_E | X_S)],
\end{equation}
which combined with \cref{lem:correlation-rounding-V-S} proves the claim. To prove \cref{eqn:averaging-C}, consider a coupling where we first sample $F \sim {V \choose k}$ and then choose $E$ uniformly at random from those subsets  $T \in \binom{V-S}{k}$ for which $F \cap S^c \subset T$. Then by symmetry, the marginal law on $E$ is uniform on ${V - S \choose k}$. Under this coupling, using the above two properties of the total correlation, we have 
\[ C(X_F | X_S) = C(X_{F \cap S^C} | X_S) \le C(X_E | X_S); \]
taking the expectation over $F$ and $E$ proves \cref{eqn:averaging-C}, and hence the result.
\end{proof}

\section{Appendix: Spectral norm of Wigner matrices}
\label{appendix:semicircle}
Let $(D_{i})_{1\leq i}$ and $(A_{ij})_{1\leq i<j\leq n}$ denote
two infinite families of independent real-valued random variables
with the following properties: 
\begin{itemize}
\item $\E[D_i] = 0 = \E[A_{ij}]$ for all $i,j$.  
\item $\E[A_{ij}^{2}] = 1$ for all $i,j$. 
\item For each $k\in \N$, $\sup_{i,j}\max\{\E[|A_{ij}|^{k}],\E[|D_i|^{k}]\} \leq C(k)<\infty$.
\end{itemize}
\begin{definition}
The random symmetric $n\times n$ matrix  $M_n$ for which $M_{ij}\sim A_{ij}$ for $1\leq i< j\leq n$ and $M_{ii} \sim D_{i}$ for $1\leq i\leq n$ is known as the \emph{$n$-dimensional real symmetric Wigner matrix} corresponding to the families distributions $(D_i)_{1\leq i}$ and $(A_{ij})_{1\leq i < j \leq n}$. 
\end{definition}
For any $n\times n$ real symmetric matrix $M_n$, we denote its eigenvalues by $\lambda_1(M_n),\dots,\lambda_{n}(M_n)$; note that these are real since $M_n$ is a real, symmetric matrix. Further, to any such matrix, we associate the following measure -- known as the \emph{rescaled empirical spectral distribution} of $M_n$ -- on the real line:
$$\mu_{M_n} := \frac{1}{n}\sum_{i=1}^{n}\delta_{\lambda_i(M_n/\sqrt{n})},$$
where $\delta_{\lambda_i(M_n/\sqrt{n})}$ denotes the Dirac delta measure supported at $\lambda_i(M_n/\sqrt{n})$.

The landmark semicircle law of Wigner gives the limit as $n \to \infty$ of the rescaled empirical spectral distribution of real symmetric $n\times n$ Wigner matrices corresponding to any families of distributions satisfying the above conditions. 
\begin{theorem}[Wigner's semicircle law, \cite{wigner1958distribution}]
Let $\mu$ denote the semicircle distribution on the real line, which is given by the density function
$$f(x) = \frac{1}{2\pi}\sqrt{4-x^{2}}\mathbf{1}_{|x|\leq 2}.$$
Fix $(D_i)_{1\leq i}$ and $(A_{ij})_{1\leq i<j}$ as above. Then, as $n\to \infty$, $\mu_{M_n}$ converges to $\mu$ in the sense that for any $\epsilon > 0$ and any continuous, bounded function $g\colon \R\to \R$:
$$\lim_{n \to \infty}\Pr_{M \sim M_n}\left[\left|\int_{\R}{gd\mu_{M} - \int_{\R}{gd\mu}}\right|>\epsilon\right]=0.$$
\end{theorem}
The Wigner matrices of interest to us are those coming from the SK spin model i.e. those corresponding to the families of distributions $D_i \sim \delta_{0}$ and $A_{ij} \sim \mathcal{N}(0,1)$ for all $i,j$. We will denote an $n\times n$ random matrix coming from this model by $M_{n,\text{SK}}$. We will make use of the semicircle law through the following immediate well-known corollary. 
\begin{corollary}
\label{corollary:spectral-norm-bound}
Let $J_n$ denote the (random) interaction matrix of the $n$-dimensional SK spin glass model. Then, for any $\delta > 0$:
$$\lim_{n\to \infty}\Pr_{J_n}\left[\|J_{n}\| \geq 2+\delta \right]= 0.$$
\end{corollary}
Here, $\|J_{n}\|$ denotes the operator norm of the matrix $J_{n}$. 
\begin{proof}
Since $J_n$ is a real symmetric matrix, $\|J\|_n = \max_{i\in [n]}\{|\lambda_i(J_n)|\}$. Hence, it suffices to show that the probability of $J_n$ having an eigenvalue outside the interval $(-2-\delta, 2+\delta)$ goes to $0$ as $n\to \infty$. Denote this event by $\mathcal{E}_{n}$, and let $g:\R\to\R$ be the piecewise linear function which is equal to $1$ on the interval $[-2,2]$ and is equal to $0$ on $(-\infty,-2-\delta/2]\cup [2+\delta/2,\infty)$. Then, since $J_n\sim M_{n,\text{SK}}/\sqrt{n}$ by definition, it follows that:
\begin{align*}
\limsup_{n\to\infty}\Pr[\mathcal{E}_{n}] & \leq\limsup_{n\to\infty}\int_{\R}(1-g)d\mu_{M_{n,\text{SK}}}\\
 & =1-\limsup_{n\to\infty}\int_{\R}gd\mu_{M_{n,\text{SK}}}\\
 & =1-\int_{\R}gd\mu+\limsup_{n\to\infty}\left(\int_{\R}gd\mu-\int_{\R}gd\mu_{M_{n,\text{SK}}}\right)\\
 & =\limsup_{n\to\infty}\left(\int_{\R}gd\mu-\int_{\R}gd\mu_{M_{n,\text{SK}}}\right)\\
 & =0,
\end{align*}
where the fourth line uses that $g$ is identically one on the support of $\mu$, and the last line uses Wigner's semicircle law. 
\end{proof}

\end{document}